\documentclass{article}
\usepackage{titletoc}

\usepackage{PRIMEarxiv}

\usepackage[utf8]{inputenc} % allow utf-8 input
\usepackage[T1]{fontenc}    % use 8-bit T1 fonts
\usepackage[natbib=true,backref=true,style=authoryear,backend=bibtex]{biblatex}
 
\usepackage[colorlinks=true,linkcolor=blue,urlcolor=magenta,citecolor=blue,anchorcolor=red]{hyperref}
\usepackage{algorithm}
\usepackage{algpseudocode}
\usepackage{amsthm}

\usepackage{url}            % simple URL typesetting
\usepackage{booktabs}       % professional-quality tables
\usepackage{amsfonts}       % blackboard math symbols
\usepackage{nicefrac}       % compact symbols for 1/2, etc.
\usepackage{microtype}      % microtypography
\usepackage{lipsum}
\usepackage{fancyhdr}       % header
\usepackage{graphicx} % graphics
\usepackage{caption} 
\graphicspath{{figures/}}     % organize your images and other figures under figures/ folder
\usepackage{subcaption}
\usepackage[inline]{enumitem}
\let\cite\relax
\newcommand*{\cite}{\citep}

\title{Federated Instrumental Variable Analysis via Federated Generalized Method of Moments}
\author{
Geetika, Somya Tyagi, Bapi Chatterjee\thanks{This work is supported in part by the Indo-French Centre for the Promotion of Advanced Research (IFCPAR/CEFIPRA) through the FedAutoMoDL project, the Infosys Center for Artificial Intelligence (CAI) at IIIT-Delhi through the Scalable Federated Learning project. Geetika is partially supported by the INSPIRE fellowship No: DST/INSPIRE Fellowship/[IF220579] offered by the Department of Science \& Technology (DST), Government of India. Bapi Chatterjee also acknowledges support by Anusandhan National Research Foundation under project SRG/2022/002269.
} \\
Department of Computer Science and Engineering, IIIT Delhi \\
New Delhi, India \\
\texttt{\{geetikai, somya23005, bapi\}@iiitd.ac.in} \\
}
\usepackage{amsmath,amsfonts,bm}

\def\eqref#1{equation~\ref{#1}}
\def\1{\bm{1}}

\def\vw{{\bm{w}}}

\def\mA{{\bm{A}}}
\def\mB{{\bm{B}}}
\def\mC{{\bm{C}}}

\def\mJ{{\bm{J}}}

\DeclareMathAlphabet{\mathsfit}{\encodingdefault}{\sfdefault}{m}{sl}
\SetMathAlphabet{\mathsfit}{bold}{\encodingdefault}{\sfdefault}{bx}{n}

\def\gA{{\mathcal{A}}}

\def\gC{{\mathcal{C}}}
\def\gD{{\mathcal{D}}}

\def\gF{{\mathcal{F}}}
\def\gG{{\mathcal{G}}}

\def\gL{{\mathcal{L}}}
\def\gM{{\mathcal{M}}}

\def\gO{{\mathcal{O}}}

\def\gR{{\mathcal{R}}}

\def\gT{{\mathcal{T}}}

\def\gZ{{\mathcal{Z}}}

\def\emLambda{{\Lambda}}

\newcommand{\E}{\mathbb{E}}

\DeclareMathOperator*{\argmax}{arg\,max}
\DeclareMathOperator*{\argmin}{arg\,min}

\usepackage{xspace,xcolor}
\definecolor{LightCyan}{rgb}{0.88,1,1}
\definecolor{Ao}{rgb}{0.0, 0.5, 0.0}
\definecolor{cadmiumorange}{rgb}{0.93, 0.53, 0.18}
\definecolor{cardinal}{rgb}{0.77, 0.12, 0.23}
\definecolor{byzantium}{rgb}{0.44, 0.16, 0.39}
\definecolor{gamboge}{rgb}{0.89, 0.61, 0.06}
\definecolor{goldenbrown}{rgb}{0.6, 0.4, 0.08}

\newcommand{\sgd}{\textsc{Sgd}\xspace}

\newcommand{\cvd}{\textsc{COVID-19}\xspace}
\newcommand{\fv}{\textsc{FedIV}\xspace}

\newcommand{\dv}{\textsc{DeepIV}\xspace}
\newcommand{\dgmm}{\textsc{DeepGMM}\xspace}
\newcommand{\agmm}{\textsc{AGMM}\xspace}

\newcommand{\oadam}{\textsc{OAdam}\xspace}

\newcommand{\sgda}{\textsc{SGDA}\xspace}
\newcommand{\fsgda}{\textsc{FedSGDA}\xspace}

\newcommand{\fiv}{\textsc{FedIV}\xspace}
\newcommand{\fgmm}{\textsc{FedGMM}\xspace}
\newcommand{\fdgmm}{\textsc{FedDeepGMM}\xspace}
\newcommand{\fgda}{\textsc{FedGDA}\xspace}

\newcommand{\gda}{\textsc{GDA}\xspace}

\newcommand{\scaf}{\textsc{Scaffold}\xspace}
\newcommand{\fprox}{\textsc{FedProx}\xspace}

\newcommand{\fopt}{\textsc{FedOpt}\xspace}

\newcommand{\tn}{\stackrel{\sim}{\smash{{\nabla}}\rule{0pt}{1.1ex}}\kern-1ex}

\makeatletter
\newcommand{\nosemic}{\renewcommand{\@endalgocfline}{\relax}}% Drop semi-colon ;
\newcommand{\dosemic}{\renewcommand{\@endalgocfline}{\algocf@endline}}% 
\let\oldnl\nl% Store \nl in \oldnl
\newcommand{\nonl}{\renewcommand{\nl}{\let\nl\oldnl}}% Remove line number for 
\newtheorem{theorem}{Theorem}%[section]
\newtheorem{definition}{Definition}%[section]
\newtheorem{lemma}{Lemma}%[section]
\newtheorem{proposition}{Proposition}%[section]
\theoremstyle{definition}
\newtheorem{remark}{Remark}%[section]
\newtheorem{assumption}{Assumption}%[section]
\makeatother

\begin{document}
\maketitle

\begin{abstract}
	Instrumental variables (IV) analysis is an important applied tool for areas such as healthcare and consumer economics. For IV analysis in high-dimensional settings, the Generalized Method of Moments (GMM) using deep neural networks offers an efficient approach. With non-i.i.d. data sourced from scattered decentralized clients, federated learning is a popular paradigm for training the models while promising data privacy. However, to our knowledge, no federated algorithm for either GMM or IV analysis exists to date. In this work, we introduce federated instrumental variables analysis (\fiv) via federated generalized method of moments (\fgmm). We formulate \fgmm as a federated zero-sum game defined by a federated non-convex non-concave minimax optimization problem, which is solved using federated gradient descent ascent (\fgda) algorithm. One key challenge arises in theoretically characterizing the federated local optimality. To address this, we present properties and existence results of clients' local equilibria via \fgda limit points. Thereby, we show that the federated solution consistently estimates the local moment conditions of every participating client. The proposed algorithm is backed by extensive experiments to demonstrate the efficacy of our approach.
\end{abstract}

% keywords can be removed

\section{Introduction}
\label{sec:intro}
Federated Learning (FL) \cite{mcmahan2017communication} over scattered clients without data sharing is now an established paradigm for training Machine Learning (ML) models. The data privacy makes it attractive for applications to healthcare \cite{nguyen2022federated,antunes2022federated,oh2023federated}, finance and banking \cite{byrd2020differentially,long2020federated}, smart cities and mobility \cite{zheng2022applications,gecer2024federated}, drug discovery \cite{oldenhof2023industry} and many others \cite{ye2023heterogeneous}. However, the existing research in FL primarily focuses on supervised learning \cite{kairouz2021advances}, which struggles to predict the outcomes due to confounding variables not observed in training data. 

For example, consider the Nature Medicine report by \citet{dayan2021federated} on a global-scale FL to predict the effectiveness of oxygen administration (a treatment variable) to \cvd patients in the emergency rooms while maintaining their privacy. It is known that \cvd revival rates are highly influenced by lifestyle-related factors such as obesity and diabetes \cite{wang2021worldwide}, other co-morbidities \cite{russell2023comorbidities}, and the patients' conditions at the emergency care admission time \cite{izcovich2020prognostic}. Arguably, the \citet{dayan2021federated}'s approach may over- or under-estimate the effects of oxygen treatment.

One can address the above issue by observing and accommodating \textit{every} confounding latent factor that may influence the outcome. Thus, it may require that obesity, diabetes, overall health at the time of admission, and even genetic factors are accommodated; for example, using a technique such as matching \cite{kallus2020generalized,kallus2020deepmatch}. It may potentially render the treatment variable undergo a randomized controlled trial such as A/B testing \cite{kohavi2013online} on decentralized, scattered, and possibly private data. However, to our knowledge, these techniques are yet unexplored in the realms of FL. 

 Alternatively, one could assume conditional independence between unobserved confounders and the treatment variable, for example, the works by   \citet{shalit2017estimating,imai2023experimental}, etc. However, this may not be a fair approach for an application such as the federated estimation of effectiveness of oxygen therapy \cite{dayan2021federated}. To elaborate, \citet{liang2023interplay} suggests the hypoxia-inducible factors (HIF) -- a protein that controls the rate of transcription of genetic information from DNA to messenger RNA by binding to a specific DNA sequence \cite{latchman1993transcription} -- plays a vital role in oxygen consumption at the cellular level. The machine learning model developed by FL implementation of \citet{dayan2021federated} would miss the crucial counterfactual scenarios, such as HIF levels among patients undergoing oxygen therapy impacting morbidity outcomes, should it assume conditional independence between effects of oxygen treatment and every confounder. Such variables can be often traced in applications such as industry-scale federated drug discovery by AstraZeneca \cite{oldenhof2023industry}.

Instrumental variables (IV) provide a workaround to both the above issues under the assumption that the latent confounding factor influences only the treatment variable but does not directly affect the outcome. In the above example, the measure of HIF works as an instrumental variable that affects oxygen treatment as in its effective organ-level consumption but does not directly affect the mortality of the \cvd patient \cite{dayan2021federated}. IV can play an important role in a federated setting as the influence assumption between the confounders and the treatment variables will remain local to the clients.

IV analysis has been comprehensively explored in econometrics \cite{angrist2001instrumental,angrist2009mostly} with several decades of history such as works of \citet{wright1928tariff} and \citet{reiersol1945confluence}.  Its efficiency is now accepted for learning even high-dimensional complex causal relationships such as one in image datasets \cite{hartford2017deep,bennett2019deep}. Naturally, the growing demand of FL entails designing methods for federated IV analysis, which, to our knowledge, is yet unexplored.

In the centralized deep learning setting, \citet{hartford2017deep} introduced an IV analysis framework, namely \dv, which uses two stages of neural networks training -- first for the treatment prediction and the second with a loss function involving integration over the conditional treatment distribution. The two-stage process has precursors in applying least square regressions in the two phases \cite{angrist2009mostly}[4.1.1]. 

In the same setting, another approach for IV analysis applies the generalized method of moments (GMM) \cite{wooldridge2001applications}. GMM is a celebrated estimation approach in social sciences and economics. It was introduced by \citet{hansen1982large}, for which he won a Nobel Prize in Economics \cite{steif2014nobel}. Building on \cite{wooldridge2001applications}, \citet{bennett2019deep} introduced deep learning models to GMM estimation; they named their method \dgmm. Empirically, \dgmm outperformed \dv. \dgmm is solved as a smooth zero-sum game formulated as a minimax optimization problem. 

Prior to \dgmm, \citet{lewis2018adversarialgeneralizedmethodmoments} also employed neural networks for GMM estimation. Their method, called the adversarial generalized method of moments (\agmm), also formulated the problem as a minimax optimization to fit a GMM criterion function over a finite set of unconditional moments. \dgmm differs from \agmm in using a weighted norm to define the objective function. The experiments in \cite{bennett2019deep} showed that \dgmm outperformed \agmm for IV analysis, and both won against \dv. Nonetheless, to our knowledge, none of these methods have a federated counterpart.  

Minimax optimization has been studied in federated settings \cite{sharma2022federated, wu2024solving}, which potentially provides an underpinning for federated GMM. However, beyond the algorithm and its convergence results, there are a few key challenges: 
\begin{enumerate}[label=(\Alph*),parsep=0pt,topsep=0pt,leftmargin=*,itemsep=0pt]
\item For non-i.i.d. client-local data, describing common federated GMM estimators is not immediate. It requires characterizing a synchronized model state that fit moment conditions of every client.
\item To show that the dynamics of federated minimax optimization retrieves an equilibrium solution of the federated zero-sum game as a limit point. And, 
\item Under heterogeneity, to establish that the federated game equilibria also satisfies the equilibrium requirements of every client thereby consistently estimating the clients' local moments.  
\end{enumerate}
In this work, we address the above challenges. Our contributions are summarized as the following:
\begin{enumerate}[parsep=0pt,topsep=0pt,leftmargin=*,itemsep=0pt]
	\item We introduce \textbf{\fv}: federated IV analysis. To our knowledge, \textbf{\fv} is the first work on IV analysis in a federated setting. 
	\item We present \textbf{\fdgmm}\footnote{\citet{wu2023personalized} used \textsc{FedGMM} as an acronym for federated Gaussian mixture models.} -- a federated adaptation of \dgmm of \citet{bennett2019deep} to solve \fv. \fdgmm is implemented as a federated smooth zero-sum game.
	\item We show that the limit points of a federated gradient descent ascent (\fgda) algorithm include the equilibria of the zero-sum game. 
    \item We show that an equilibrium solution of the federated game obtained at the server consistently estimates the moment conditions of every client.  	
	\item We experimentally validate our algorithm. The experiments show that even for heterogenous data, \fdgmm has convergent dynamics analogous to the centralized \dgmm algorithm.   
\end{enumerate} 

\subsection{Related work}

The federated supervised learning has received algorithmic advancements guided by factors such as tackling the system and statistical heterogeneities, better sample and communication complexities, model personalization, differential privacy, etc. An inexhaustible list includes \fprox \cite{li2020federated}, \scaf \cite{karimireddy2020scaffold}, \fopt \cite{reddi2020adaptive}, \textsc{LPP-SGD} \cite{chatterjee2024federated}, \textsc{pFedMe} \cite{t2020personalized}, \textsc{DP-SCAFFOLD} \cite{noble2022differentially}, and others.

By contrast, federated learning with confounders, which typically forms a causal learning setting, is a relatively under-explored research area. \citet{vo2022adaptive} presented a method to learn the similarities among the data sources translating a structural causal model \cite{pearl2009causal} to federated setting. They transform the loss function by utilizing Random Fourier Features into components associated with the clients. Thereby they compute individual treatment effects (ITE) and average treatment effects (ATE) by a federated maximization of evidence lower bound (ELBO). \citet{vo2022bayesian} presented another federated Bayesian method to estimate the posterior distributions of the ITE and ATE using a non-parametric approach. 

\citet{xiong2023federated} presented maximum likelihood estimator (MLE) computation in a federated setting for ATE estimation. They showed that the federated MLE consistently estimates the ATE parameters considering the combined data across clients. %Their model incorporates heterogeneity in the treatment effects across clients. 
However, it is not clear if this approach is applicable to consistent local moment conditions estimation for the participating clients. 
\citet{almodovar2024propensity} applied FedAvg to variational autoencoder \cite{kingma2019introduction} based treatment effect estimation TEDVAE \cite{zhang2021treatment}. However, their work mainly focused on comparing the performance of vanilla FedAvg with a propensity score-weighted FedAvg in the context of federated implementation of TEDVAE. 

Our work differs from the above related works in the following: 
\begin{enumerate}[label=(\alph*),parsep=0pt,topsep=0pt,leftmargin=*,itemsep=0pt]
	\item we introduce IV analysis in federated setting, and, we introduce federated GMM estimators, which has applications for various empirical research \cite{wooldridge2001applications}, 
	\item specifically, we adopt a non-Bayesian approach based on a federated zero-sum game, wherein we focus on analysing the dynamics of the federated minimax optimization and characterize the global equilibria as a consistent estimator of the clients' moment conditions.
\end{enumerate} 

Our work also differs from federated minimax optimization algorithms: \citet{sharma2022federated, shen2024stochastic, wu2024solving, zhu2024stability}, where the motivation is to analyse and improve the non-asymptotic convergence under various analytical assumptions on the objective functions. We primarily focus on deriving the equilibrium via the limit points of the federated GDA algorithm.

\section{Preliminaries}
\label{sec:prelim}		
We model our basic terminologies after \cite{bennett2019deep} for a client-local setting. Consider a distributed system as a set of $N$ clients $[N]$ with datasets $S^i = \{(x_j^i,y_j^i)\}_{j=1}^{n_i},~\forall i\in[N]$. We assume that for a client $i\in[N]$, the treatment and outcome variables $x_j^i\text{ and }y_j^i$, respectively, are related by the process $
	Y^i=g_0^i(X^i) + \epsilon^i, ~i \in [N].
$
We assume that each client-local residual $\epsilon^i$ has zero mean and finite variance, i.e. $
\E[\epsilon^i] = 0, \E[({\epsilon^i})^2] < \infty. 
$
Furthermore, we assume that the treatment variables $X^i$ are endogenous on the clients, i.e. $
\E[\epsilon^i|X^i] \ne 0,\text{ and therefore, } g_0^i(X^i) \ne \E[Y^i|X^i].
$

We assume that the treatment variables are influenced by instrumental variables  $Z^i, \forall 	i\in[N]$ so that  
\begin{equation}\label{eqn:IV1}
	P(X^i|Z^i)\ne P(X^i).
\end{equation}
Furthermore, the instrumental variables do not directly influence the outcome variables $Y^i, \forall i\in[N]$:
\begin{equation}\label{eqn:IV2}
\E[\epsilon^i|Z^i] = 0.
\end{equation}
Note that, assumptions \ref{eqn:IV1}, \ref{eqn:IV2} are local to the clients, thus, honour the data-privacy requirements of a federated learning task. In this setting, we aim to discover a common or global causal response function that would fit the data generation processes of each client without centralizing the data. More specifically, we learn a parametric function $g_0(.)\in G:=\{g(.,\theta)|\theta\in\Theta\}$ expressed as $g_0:=g(.,\theta_0)$ for $\theta_0\in \Theta$, defined by
\begin{equation}\label{eqn:fedg}
	g(.,\theta_0) = \frac{1}{N}\sum_{i=1}^{N}g^i(.,\theta_0).
\end{equation}
The learning process essentially involves estimating the true parameter $\theta_0$ by $\hat{\theta}$. To measure the performance of the learning procedure, we use the MSE of the estimate $\hat{g}:=g(.,\hat{\theta})$ against the true $g_0$ averaged over the clients.

\section{Federated Deep Generalized Method of Moments}
\label{sec:desc}		
We adapt \dgmm \cite{bennett2019deep} in the local setting of a client $i\in[N]$. For a self-contained reading, we include the description here.
\subsection{Client-local Deep Generalized Method of Moments (\dgmm)}
\label{sec:dgmm}
GMM estimates the parameters of the causal response function using a certain number of \textit{moment conditions}. Define the \textit{moment function} on a client $i\in[N]$ as a vector-valued function $f^i:\mathbb{R}^{|\gZ|}\rightarrow\mathbb{R}^{m}$ with components $f_1^i,f_2^i,\dots, f_m^i$. We consider the moment conditions as parametrized functions $\{f_j^i\}_{j=1}^m~\forall i\in[N]$ with the assumption that their expectation is zero at the true parameter values. More specifically, using equation (\ref{eqn:IV2}), we have 
\begin{equation}\label{eq:exog1}
	\E[f_j^i(Z^i)\epsilon^i] = 0, \forall j \in [m],~\forall i\in[N], 
\end{equation}
We assume that $m$ moment conditions $\{f_j^i\}_{j=1}^m$ at each client $i\in[N]$ are sufficient to identify a unique federated estimate $\hat{\theta}$ to $\theta_0$. With (\ref{eq:exog1}), we define the moment conditions on a client $i\in[N]$ as 
\begin{equation}\label{eq:exog2}
	\psi(f^i_j;\theta)=0,~\forall j\in[m],\text{ where }
\end{equation} 
\[
\psi(f^i;\theta)=\mathbb{E}[f^i(Z^i)\epsilon^i]=\mathbb{E}[f^i(Z^i)(Y^i-g^i(X^i;\theta)).
\]
In empirical terms, the sample moments for the $i$-th client with $n_i$ samples are given by
\begin{equation}
	\psi_{n_i}(f^i;\theta)=\mathbb{E}_{n_i}[f^{i}(Z)\epsilon ^i]=\frac{1}{n_i}\sum_{k=1}^{n_i}f^i(Z_k^i)(Y^i_k-g^i(X_k^i;\theta)),
\end{equation}
where $\psi_{n_i}(f^i;\theta)=\left(\psi_{n_i}(f^i_1;\theta),\psi_{n_i}(f^i_2;\theta),\ldots, \psi_{n_i}(f^i_m;\theta)\right)$ is the moment condition vector, and 
\begin{equation}
	\psi_{n_i}(f^i_j;\theta)=\frac{1}{n_i}\sum_{k=1}^{n_i}f^i_j(Z_k^i)(Y^i_k-g^i(X_k^i;\theta)).
\end{equation}
Thus, for empirical estimation of the causal response function $g_0^i$ at client $i\in[N]$, it needs to satisfy 
\begin{equation}\label{eqn:MC1}
	\psi_{n_i}(f_{j}^{i};\theta_0)=0,~ \forall~ i\in[N] \text{ and } j\in[m]
\end{equation}
at $\theta=\theta_0$. Equation (\ref{eqn:MC1}) is reformulated as an optimization problem given by 
\begin{equation}\label{eqn:MC2}
	\min_{\theta\in\Theta}\|\psi_{n_i}(f^i_1;\theta),\psi_{n_i}(f^i_2;\theta),\ldots, \psi_{n_i}(f^i_m;\theta)\|^2,
\end{equation}
where we use the Euclidean norm $\|w\|^2=w^Tw$. Drawing inspiration from \citet{hansen1982large}, \dgmm used a weighted norm, which yields minimal asymptotic variance for a consistent estimator $\tilde{\theta}$, to cater to the cases of (finitely) large number of moment conditions. We adapt their weighted norm $\|w\|^2_{\tilde{\theta}}=w^T\gC_{\tilde{\theta}}^{-1}w$, to a client-local setting via the covariance matrix $\gC_{\tilde{\theta}}$ defined by
\begin{equation}\label{eqn:MC3}
	\left[\gC_{\tilde{\theta}}\right]_{jl}=\frac{1}{n_i}\sum_{k=1}^{n_i}f^i_j(Z^i_k)f^i_l(Z^i_k)(Y^i_k-g^i(X^i_k;\tilde{\theta}))^2.
\end{equation}
Now considering the vector space $\mathcal{V}$ of real-valued functions, $\psi_{n_i}(f^i;\theta)=\left(\psi_{n_i}(f^i_1;\theta),\psi_{n_i}(f^i_2;\theta),\ldots, \psi_{n_i}(f^i_m;\theta)\right)$ is a linear operator on $\mathcal{V}$ and 
\begin{equation}\label{eqn:MC4}
	\gC_{\tilde{\theta}}(f^i,h^i)=\frac{1}{n_i}\sum_{k=1}^{n_i}f^i(Z^i_k)h^i(Z^i_k)(Y^i_k-g^i(X^i_k;\tilde{\theta}))^2
\end{equation}
is a bilinear form. With that, for any subset $\mathcal{F}^i\subset\mathcal{V}$, we define a function 
\begin{equation*}
	\Psi_{n_i}(\theta,\mathcal{F}^i,\tilde{\theta})= \sup_{f^i\in\mathcal{F}^i} \psi_{n_i}(f^i;\theta)-\frac{1}{4}\gC_{\tilde{\theta}}(f^i,f^i),
\end{equation*}
which leads to the following optimization problem.
\begin{lemma}[Lemma 1 of \cite{bennett2019deep}]
	\label{lemma:U} 
	With the weighted norm defined by equation (\ref{eqn:MC3}), and for $\mathcal{F}^i=span(\{f^i_j\}_{j=1}^m)$
	\begin{equation}\label{eqn:MC6}
		\|\psi_{n_i}(f^i_1;\theta),\psi_{n_i}(f^i_2;\theta),\ldots, \psi_{n_i}(f^i_m;\theta)\|^2_{\tilde{\theta}}=\Psi_{n_i}(\theta,\mathcal{F}^i,\tilde{\theta}).
	\end{equation}
	Thus, a weighted reformulation of (\ref{eqn:MC2}) is given by
	\begin{equation}\label{eqn:MC7}
		\theta^{\text{GMM}}\in\argmin_{\theta \in \Theta}\Psi_{n_i}(\theta,\mathcal{F}^i,\tilde{\theta}).
	\end{equation}	
\end{lemma}
As the data-dimension grows, the function class $\mathcal{F}^i$ is replaced with a class of neural networks of a certain architecture, i.e. $\mathcal{F}^i=\{f^i(z,\tau):\tau\in\mathcal{T}\}$. Similarly, let $\mathcal{G}^i=\{g^i(x,\theta):\theta\in\Theta\}$ be another class of neural networks with varying weights. With that, define 
\begin{equation}\label{eqn:MC8}
	\begin{aligned}
		U^i_{\tilde{\theta}}(\theta,\tau):= \frac{1}{n_i}\sum_{k=1}^{n_i}f^i(Z^i_k,\tau)\left(Y^i_k-g^i(X^i_k;\theta)\right)
		-\frac{1}{4n_i}\sum_{k=1}^{n_i}\left(f^i(Z^i_k,\tau)\right)^2\left(Y^i_k-g^i(X^i_k;\theta)\right)^2
	\end{aligned}  
\end{equation}
Then (\ref{eqn:MC7}) is reformulated as the following
\begin{equation}\label{eqn:MC9}
	\begin{aligned}
		\theta^{\text{DGMM}}\in\argmin_{\theta \in \Theta}\sup_{\tau\in\mathcal{T}}U^i_{\tilde{\theta}}(\theta,\tau).
	\end{aligned}
\end{equation}
Equation (\ref{eqn:MC9}) forms a zero-sum game, whose equilibrium solution is shown to be a true estimator to $\theta_0$ under a set of standard assumptions; see Theorem 2 in \cite{bennett2019deep}.

\subsection{Federated Deep GMM (\fdgmm)}
The federated generalized method moment (\fdgmm) needs to find the global moment estimators for the causal response function to fit data on each client. Thus, the federated counterpart of equation (\ref{eq:exog2}) is given by  
\begin{equation}\label{eqn:MC10}
\psi(f;\theta)=\mathbb{E}_i[\mathbb{E}[f^i(Z^i)(Y^i_k-g^i(X^i;\theta)]]=0,
\end{equation}
where the expectation $\mathbb{E}_i$ is over the clients. In this work, we consider \textit{full client participation}. Thus, for the empirical federated moment estimation, we formulate:
\begin{equation}
\begin{aligned}
\psi_n(f;\theta)&=\frac{1}{N}\sum_{i=1}^{N} \psi_{n_i}(f^i;\theta)=\frac{1}{N}\sum_{i=1}^{N}\frac{1}{n_i}\sum_{k=1}^{n_i}f^i(Z_k^i)(Y^i_k-g^i(X_k^i;\theta))
\end{aligned}
\end{equation}
With that, the federated moment estimation problem following (\ref{eqn:MC7}) is formulated as:
\begin{equation}\label{eqn:FMC}
   {\theta}^{\text{FedDeepGMM}} \in \arg\min_{\theta \in \Theta} {\left\|  \psi_{n}(f; \theta)\right\|_{\tilde{\theta}}^2 },
\end{equation}
where $\|w\|_{\tilde{\theta}}=w^\top \gC_{\tilde{\theta}}^{-1} x$ is the previously defined weighted-norm with inverse covariance as weights.
In general cases, we do not have explicit knowledge of the moment conditions of various clients. We propose \fdgmm, a ``deep" reformulation of the federated optimization problem based on the neural networks of a given architecture shared among clients and is shown to have the same solution as the federated GMM problem formulated earlier.
\begin{lemma}
\label{lemma:federatedU}
Let $\gF = \text{span}\{f^i_j~\vert~i\in[N],~j\in[m]\}$. An equivalent objective function for the federated moment estimation optimization problem (\ref{eqn:FMC}) is given by:
\begin{equation}
\left\|  \psi_{N}(f; \theta)\right\|_{\tilde{\theta}}^2 =\sup_{\substack{f^i\in \mathcal{F}\\ \forall i\in[N]}}\frac{1}{N}\sum_{i=1}^{N} \left(\psi_{n_i}(f^i;\theta)-\frac{1}{4}\gC_{\tilde{\theta}}(f^i;f^i)\right), \text{ where}
\end{equation}
\[\begin{aligned}
&\psi_{n_i}(f^i; \theta) :=  \frac{1}{n_i} \sum_{k=1}^{n_i} f^i(Z_k^i)(Y^i_k - g^i(X^i_k;\theta)), \text{ and }
\mathcal{C}_{\tilde{\theta}}(f^i, f^i) := \frac{1}{n_i} \sum_{k=1}^{n_i} (f^i(Z^i_k))^2 (Y_k^i - g^i(X_k^i;\tilde{\theta}))^2.
\end{aligned}
\]
\end{lemma}
The detailed proof is similar to Lemma~\ref{lemma:U} and is given in Appendix~\ref{ap::lemma:federatedU}. The federated zero-sum game is then defined by:
\begin{align}
\label{eqn:fedoptprob}
\hat{\theta}^{\text{FedDeepGMM}} \in \argmin_{\theta \in \Theta} \sup_{\tau \in \mathcal{T}} U_{\tilde{\theta}}(\theta, \tau) := \frac{1}{N} \sum_{i=1}^N U_{\tilde{\theta}}^i(\theta, \tau),
\end{align}
where 
$U_{\tilde{\theta}}^i(\theta, \tau)$ is defined in equation~(\ref{eqn:MC8}).
The federated GMM formulation by a zero-sum game defined by a federated minimax optimization problem (\ref{eqn:fedoptprob}) provides the global estimator as its equilibrium solution. We solve (\ref{eqn:fedoptprob}) using the federated gradient descent ascent (\fgda) algorithm described next. 
\subsection{Federated Gradient Descent Ascent (\fgda) Algorithm}
An adaptation of the standard gradient descent ascent algorithm to federated setting is well-explored: \cite{deng2021local,sharma2022federated,shen2024stochastic,wu2024solving}. The clients run the gradient descent ascent algorithm for several local updates and then the orchestrating server synchronizes them by collecting the model states, averaging them, and broadcasting it to the clients. A detailed description is included as a pseudocode in Appendix \ref{sec:algo}.  

Similar to \cite{bennett2019deep}, we note that the federated minimax optimization problem (\ref{eqn:fedoptprob}) is not convex-concave on $(\theta, \tau)$. The convergence results of variants of \fgda \cite{,sharma2022federated,shen2024stochastic,wu2024solving} assume that $U_{\tilde{\theta}}(\theta, \tau)$ is non-convex on $\theta$ and satisfies a $\mu-$Polyak Łojasiewicz (PL) inequality on $\tau$, see assumption 4 in \cite{sharma2022federated}. PL condition is known to be satisfied by over-parametrized neural networks \cite{charles2018stability, liu2022loss}. The convergence results of our method will follow \cite{sharma2022federated}. We include a formal statement in Appendix \ref{sec:algo}. However, beyond convergence, we primarily aim to show that an optimal solution will consistently estimate the moment conditions of the clients, which we do next.

\section{Federated Equilibrium Solutions}
\label{sec:res}		
In this section, we present our main results, which establish the existence and characterize the federated equilibrium solution. 
\subsection{Federated Sequential Game}
As minimax is not equal to maximin in general for a non-convex-non-concave problem, it is important to model the federated game as a sequential game \cite{pmlr-v119-jin20e} whose outcome would depend on what move -- maximization or minimization -- is taken first. We use some results from \citet{pmlr-v119-jin20e}, which we include here for a self-contained reading. We start with the following assumptions:
\begin{assumption}
\label{asm:twicediff}
    Client-local objective $U_{\tilde\theta}^i(\theta,\tau)$ $\forall i \in [N]$ is twice continuously differentiable for both $\theta$ and $\tau$.
     Thus, the global objective $U_{\tilde\theta}(\theta,\tau)$ is also a twice continuously differentiable function.
\end{assumption}
\begin{assumption}[Smoothness]
\label{asm:lipschitz}
    The gradient of each client's local objective, $ \nabla U_{\tilde\theta}^i(\theta,\tau) $, is Lipschitz continuous with respect to both $\theta$ and $\tau$. For all $ i \in [N] $, there exist constants $L> 0$ such that:
    \begin{align*}
          &\|\nabla_\theta U_{\tilde\theta}^i(\theta_1, \tau_1) - \nabla_\theta U_{\tilde\theta}^i(\theta_2, \tau_2)\|\leq L \|(\theta_1,\tau_1)-(\theta_2,\tau_2)\|,\text{ and}\\
    	&\|\nabla_\tau U_{\tilde\theta}^i(\theta_1, \tau_1) - \nabla_\tau U_{\tilde\theta}^i(\theta_2, \tau_2)\|\leq L \|(\theta_1,\tau_1)-(\theta_2,\tau_2)\|, 
    \end{align*}
    $\forall (\theta_1, \tau_1), (\theta_2, \tau_2)$. Thus, $U_{\tilde\theta}(\theta,\tau) $ is $L$-Lipschitz smooth.
\end{assumption}
\begin{assumption}[Gradient Dissimilarity]
\label{asm:gradientheterogeneity}
The heterogeneity of the local gradients with respect to (w.r.t.) $\theta$ and $\tau$ is bounded as follows:
    \begin{align*}
        \|\nabla_{\theta}U_{\tilde\theta}^i(\theta,\tau)-\nabla_{\theta}U_{\tilde\theta}(\theta,\tau)\|\leq \zeta_{\theta}^i\qquad\qquad
        \|\nabla_{\tau}U_{\tilde\theta}^i(\theta,\tau)-\nabla_{\tau}U_{\tilde\theta}(\theta,\tau)\|\leq \zeta_{\tau}^i,
    \end{align*}
    where $\zeta_{\theta}^i,~\zeta_{\tau}^i\geq0$ are the bounds that quantify the degree of gradient dissimilarity at client $ i \in [N] $.
\end{assumption}
\begin{assumption}[Hessian Dissimilarity]
\label{asm:hessianheterogeneity}
    The heterogeneity in terms of hessian w.r.t. $\theta$ and $\tau$ is bounded as follows:
    \begin{align*}
   &\|\nabla_{\theta\theta}^2U_{\tilde\theta}^i(\theta,\tau)-\nabla_{\theta\theta}^2U_{\tilde\theta}(\theta,\tau)\|_{\sigma}\leq \rho_{\theta}^i,\qquad&
&\|\nabla_{\tau\tau}^2U_{\tilde\theta}^i(\theta,\tau)-\nabla_{\tau\tau}^2U_{\tilde\theta}(\theta,\tau)\|_{\sigma}\leq \rho_{\tau}^i,\\
&\|\nabla_{\theta\tau}^2U_{\tilde\theta}^i(\theta,\tau)-\nabla_{\theta\tau}^2U_{\tilde\theta}(\theta,\tau)\|_{\sigma}\leq \rho_{\theta\tau}^i,\qquad&
&\|\nabla_{\tau\theta}^2U_{\tilde\theta}^i(\theta,\tau)-\nabla_{\tau\theta}^2U_{\tilde\theta}(\theta,\tau)\|_{\sigma}\leq \rho_{\tau\theta}^i,
    \end{align*}
    where $\rho_{\theta}^i,~\rho_{\tau}^i,~ \rho_{\theta\tau}^i,~\text{and }\rho_{\tau\theta}^i\geq0$ quantify the degree of hessian dissimilarity at client $ i \in [N] $ by spectral norm $\|.\|_{\sigma}$.
\end{assumption}
Assumptions~\ref{asm:gradientheterogeneity} and~\ref{asm:hessianheterogeneity} provide a measure of data heterogeneity across clients in a federated setting. We assume that $\zeta's$ and $\rho's$ are bounded. In the special case, when $\zeta$ and $\rho$'s are all $0$, then the data is homogeneous across clients.

We adopt the notion of Stackelberg equilibrium for pure strategies, as discussed in \cite{pmlr-v119-jin20e}, to characterize the solution of the minimax federated optimization problem for a non-convex non-concave function $U_{\tilde\theta}(\theta,\tau)$ for the sequential game where min-player goes first and the max-player goes second.

To avoid ambiguity between the adjectives of the terms global/local objective functions in federated learning and the global/local nature of minimax points in optimization, we refer to a global objective as the federated objective and a local objective as the client's objective.
\begin{definition}[Local minimax point][Definition 14 of \cite{pmlr-v119-jin20e}]
	\label{def:localequilibrium}
	Let $U(\theta, \tau)$ be a function defined over $\Theta \times \mathcal{T} $ and let $h$ be a function satisfying $h(\delta)\rightarrow0$ as $\delta\rightarrow0$. There exists a $\delta_0$, such that for any $\delta\in(0,\delta_0],$ and any $(\theta,\tau)$ such that $\|\theta-\hat\theta\|\leq \delta$ and $\|\tau-\hat\tau\|\leq\delta$, then a point $(\hat\theta,\hat\tau)$ is a local minimax point of $U$, if $\forall~(\theta,\tau)\in\Theta\times\gT$, it satisfies:
		\begin{equation}
			\label{eqn:localequilibrium}
			U_{\tilde\theta}(\hat\theta,\tau)\leq U_{\tilde\theta}(\hat\theta,\hat\tau)\leq \max_{{\tau\prime : \|\tau\prime-\hat\tau\|\leq h(\delta)}}U_{\tilde\theta}(\theta,\tau\prime),
		\end{equation}
\end{definition}
With that, the first-order \& second-order necessary conditions for local minimax points are as below.
\begin{lemma} [Propositions 18, 19, 20 of \cite{pmlr-v119-jin20e}] 
    Under assumption~\ref{asm:twicediff}, any local minimax point satisfies the following conditions:
    \begin{itemize}[parsep=0pt,topsep=0pt,leftmargin=*,itemsep=0pt]
        \item \textbf{First-order Necessary Condition:} A local minimax point $(\theta,\tau)$ satisfies: $\nabla_{\theta}U_{\tilde\theta}(\theta,\tau)=0$ and $\nabla_{\tau}U_{\tilde\theta}(\theta,\tau)=0$.
         \item \textbf{Second-order Necessary Condition:} A local minimax point $(\theta,\tau)$ satisfies: $\nabla_{\tau\tau}^2U_{\tilde\theta}(\theta,\tau)\preceq \mathbf{0}.$ Moreover, if $\nabla_{\tau\tau}^2U_{\tilde\theta}(\theta,\tau)\prec 0$, then
         \begin{equation*}
             \left[\nabla_{\theta\theta}^2U_{\tilde\theta}-\nabla_{\theta\tau}^2U_{\tilde\theta}\left(\nabla_{\tau\tau}^2U_{\tilde\theta}\right)^{-1}\nabla_{\tau\theta}^2U_{\tilde\theta}\right](\theta,\tau)\succeq 0.
    \end{equation*}
        \item \textbf{Second-order Sufficient Condition:} A stationary point $(\theta,\tau)$ that satisfies $\nabla_{\tau\tau}^2U_{\tilde\theta}(\theta,\tau)\prec \mathbf{0}$, and
         \begin{equation*}
			\left[\nabla_{\theta\theta}^2U_{\tilde\theta}-\nabla_{\theta\tau}^2U_{\tilde\theta}\left(\nabla_{\tau\tau}^2U_{\tilde\theta}\right)^{-1}\nabla_{\tau\theta}^2U_{\tilde\theta}\right](\theta,\tau){\succ} 0
		\end{equation*}
        guarantees that $(\theta,\tau)$ is a strict local minimax. 
    \end{itemize}
\end{lemma}
Now, in order to define the federated approximate equilibrium solutions, we first define an approximate local minimax point.
\begin{definition}[Approximate Local minimax point][An adaptation of definition 34 of \cite{pmlr-v119-jin20e}]
	\label{def:approxepsilonlocalequilibrium}
	Let $U(\theta, \tau)$ be a function defined over $\Theta \times \mathcal{T} $ and let $h$ be a function satisfying $h(\delta)\rightarrow0$ as $\delta\rightarrow0$. There exists a $\delta_0$, such that for any $\delta\in(0,\delta_0],$ and any $(\theta,\tau)$ such that $\|\theta-\hat\theta\|\leq \delta$ and $\|\tau-\hat\tau\|\leq\delta$, then a point $(\hat\theta,\hat\tau)$ is an $\varepsilon$-approximate local minimax point of $U$, if it satisfies:
		\begin{equation}
			\label{eqn:approxlocalequilibrium}
			U_{\tilde\theta}(\hat\theta,\tau)-\varepsilon\leq U_{\tilde\theta}(\hat\theta,\hat\tau)\leq {\max_{{\tau\prime : \|\tau\prime-\hat\tau\|\leq h(\delta)}}U_{\tilde\theta}(\theta,\tau\prime)}+\varepsilon,
		\end{equation}
\end{definition}
We aim to achieve approximate local minimax points for every client as a solution of the federated minimax optimization. With that, we characterize the federated solution as the following. 
\begin{definition}[$\mathcal{E}$-Approximate Federated Equilibrium Solutions]\label{def:epsilonequilibrium}    	  
Let $\mathcal{E}=\{\varepsilon^i\}^N_{i=1}$ be the approximation error vector for clients $[N] $. Let $ U_{\tilde\theta}^i(\theta, \tau) $ be a function defined over $ \Theta \times \mathcal{T} $ for a client $i\in[N]$. An $\mathcal{E}$-approximate federated equilibrium point $ (\hat\theta, \hat\tau) $ that is an $\varepsilon^i$-approximate local minimax point for every clients' objective $U_{\tilde{\theta}}^i$, where the federated objective is $U_{\tilde{\theta}}(\theta, \tau) := \frac{1}{N} \sum_{i=1}^N U_{\tilde{\theta}}^i(\theta, \tau)$, must follow the conditions below:
\begin{enumerate}[parsep=0pt,topsep=0pt,leftmargin=*,itemsep=0pt]
\label{def:eps_equilibrium}
    \item \textbf{$\varepsilon^i$- First-order Necessary Condition:}  
    The point $ (\hat\theta, \hat\tau) $ must be an $ \varepsilon^i $ stationary point for every client $i\in[N]$, i.e.,  
    \begin{equation*}
	    \|\nabla_{\theta} U_{\tilde\theta}^i(\hat\theta, \hat\tau)\| \leq \varepsilon^i, 
        \quad \text{ and } \quad 
        \|\nabla_{\tau} U_{\tilde\theta}^i(\hat\theta, \hat\tau)\| \leq \varepsilon^i.
    \end{equation*}  
    
    \item \textbf{Second-Order $\varepsilon^i$ Necessary Condition:}  
    The point $ (\hat\theta, \hat\tau) $ must satisfy the second-order conditions:  
    \begin{align*}
     \nabla^2_{\tau\tau} U^i_{\tilde\theta}(\hat\theta, \hat\tau) \preceq -\varepsilon^i I, 
        \quad\text{and}\quad \left[\nabla_{\theta\theta}^2U_{\tilde\theta}^i-\nabla_{\theta\tau}^2U_{\tilde\theta}^i\left(\nabla_{\tau\tau}^2U_{\tilde\theta}\right)^{-1}\nabla_{\tau\theta}^2U_{\tilde\theta}^i\right](\hat\theta,\hat\tau) \succeq \varepsilon^i I.
    \end{align*}  
    \item \textbf{Second-Order $\varepsilon^i$ Sufficient Condition:} An $\varepsilon^i$ stationary point $(\theta,\tau)$ that satisfies $\nabla_{\tau\tau}^2U_{\tilde\theta}^i(\hat\theta,\hat\tau)\prec-\varepsilon^i I$, and
         \begin{equation*}
		\left[\nabla_{\theta\theta}^2U_{\tilde\theta}-\nabla_{\theta\tau}^2U_{\tilde\theta}\left(\nabla_{\tau\tau}^2U_{\tilde\theta}\right)^{-1}\nabla_{\tau\theta}^2U_{\tilde\theta}\right](\hat\theta,\hat\tau)\succ \varepsilon^i I~~
        \end{equation*}
        guarantees that $(\hat\theta,\hat\tau)$ is a strict local minimax point $\forall i\in[N]$ that satisfies $\varepsilon^i$ approximate equilibrium as in definition~\ref{def:approxepsilonlocalequilibrium}.
\end{enumerate}
\end{definition}
We now state the main theoretical result of our work in the following theorem.    
    \begin{theorem}
    \label{sec::thm:globallocalequilibrium}
        Under assumptions~\ref{asm:twicediff},~\ref{asm:lipschitz},~\ref{asm:gradientheterogeneity} and~\ref{asm:hessianheterogeneity}, a minimax solution $(\hat\theta,\hat\tau)$ of federated optimization problem~(\ref{eqn:fedoptprob}) that satisfies the equilibrium condition as in definition~\ref{def:localequilibrium}: 
          \begin{equation*}
	     U_{\tilde\theta}(\hat\theta,\tau)\leq U_{\tilde\theta}(\hat\theta,\hat\tau)\leq \max_{{\tau\prime : \|\tau\prime-\hat\tau\|\leq h(\delta)}}U_{\tilde\theta}(\theta,\tau\prime),
		\end{equation*}
     is an $\mathcal{E}$-approximate federated equilibrium solution as defined in \ref{def:epsilonequilibrium}, where the approximation error $\varepsilon^i$ for each client $i\in[N]$ lies in:
\begin{equation*}
\max\{ \zeta_{\theta}^i, \zeta_{\tau}^i\}\le \varepsilon^i\le \min\{\alpha-\rho_{\tau}^i,\beta- B^i \}    
\end{equation*} 
 for $\rho_{\tau}^i<\alpha$ and $B^i>\beta$, such that $\alpha:=\left\vert\lambda_{\max}\left(\nabla_{\tau\tau}^2U_{\tilde\theta}(\hat\theta,\hat\tau)\right)\right\vert$, $\beta:=\lambda_{\min}\left(\left[\nabla_{\theta\theta}^2U_{\tilde\theta}-\nabla_{\theta\tau}^2U_{\tilde\theta}\left(\nabla_{\tau\tau}^2U_{\tilde\theta}\right)^{-1}\nabla_{\tau\theta}^2U_{\tilde\theta}\right](\hat\theta,\hat\tau)\right)$ and $B^i:=\rho_{\theta}^i+ L\rho_{\theta\tau}^i\frac{1}{\vert\lambda_{\max}(\nabla_{\tau\tau}^2 U_{\tilde\theta}^i)\vert}+L\rho_{\tau\theta}^i\frac{1}{\vert\lambda_{\max}(\nabla_{\tau\tau}^2 U_{\tilde\theta}^i)\vert}+ L^2\rho_{\tau}^i\frac{1}{\vert\lambda_{\max}(\nabla_{\tau\tau}^2 U_{\tilde\theta}^i)\cdot\lambda_{\max}(\nabla_{\tau\tau}^2 U_{\tilde\theta})\vert}$.
    \end{theorem}
    The proof of theorem \ref{sec::thm:globallocalequilibrium} is given in Appendix~\ref{ap::thm:globallocalequilibrium}. Note that when data is homogeneous (i.e., for each client $i$, $\zeta_{\theta}^i$, $\zeta_{\tau}^i$, $\rho_{\tau}^i$ and $B^i$ are all zeroes), each client satisfies an exact local minimax equilibrium.
    \begin{remark}
       In Theorem~\ref{sec::thm:globallocalequilibrium}, note that if the interval $[\max\{ \zeta_{\theta}^i, \zeta_{\tau}^i\}, \min\{\alpha-\rho_{\tau}^i,\beta- B^i \}]$ is empty, i.e. $\max\{ \zeta_{\theta}^i, \zeta_{\tau}^i\}>\min\{\alpha-\rho_{\tau}^i,\beta- B^i \}$,  then no such $\varepsilon^i$ exists and $(\hat\theta,\hat\tau)$ fails to be a local $\varepsilon^i$ approximate equilibrium point for that clients. It may happen in two cases:
       \begin{enumerate}[parsep=0pt,topsep=0pt,leftmargin=*,itemsep=0pt]
           \item The gradient dissimilarity $\zeta_{\theta}^i, \zeta_{\tau}^i$ is too large indicating high heterogeneity, then $(\hat\theta,\hat\tau)$- the solution to the federated objective would fail to become an approximate equilibrium point for the clients. It is a practical consideration for a federated convergence facing difficulty against high heterogeneity. 
           \item If $\alpha\approx\rho_{\tau}^i$ or $\beta\approx B^i$, indicating that the client’s local curvature structure significantly differs from the global curvature. In this case, the clients' objectives may be flatter or even oppositely curved compared to the global model, that is, the objectives are highly heterogeneous.
       \end{enumerate} 
    \end{remark}
    
    Now we state the result on the consistency of the estimator of the clients' moment conditions.

\begin{theorem}[Consistency][Adaptation of Theorem 2 of \cite{bennett2019deep}]
\label{thm:consistency}
    Let $\tilde\theta_{n}$ be a data-dependent choice for the federated objective that has a limit in probability. For each client $i\in[N]$, define $        m^i(\theta,\tau,\tilde{\theta}):=f^i(Z^i;\tau)(Y^i-g(X^i;\theta))-\frac{1}{4}f^i(Z^i;\tau)^2(Y^i-g(X^i;\tilde{\theta}))^2$, $ M^i(\theta)=\sup_{\tau \in \gT}\mathbb{E}[m^i(\theta,\tau,\tilde{\theta})]$ and  $\eta^i(\epsilon):=inf_{d(\theta,\theta_0)\ge\epsilon}M^i(\theta)-M^i(\theta_0)$ for every $\epsilon>0$. Let $(\hat\theta_{n},\hat\tau_{n})$ be a solution that satisfies the approximate equilibrium for each of the client $i\in[N]$ as
\begin{eqnarray*}
      \sup_{\tau\in\gT}  U_{\tilde\theta}^i(\hat\theta_{n},\tau)-\varepsilon^i - o_p(1)\leq  ~U^i_{\tilde\theta}(\hat\theta_{n},\hat\tau_{n})\leq ~\inf_{\theta\in\Theta} {\max_{{\tau\prime : \|\tau\prime-\hat\tau_{n}\|\leq h(\delta)}} U^i_{\tilde\theta}(\theta,\tau\prime)}+\varepsilon^i+o_p(1),
\end{eqnarray*}
     for some $\delta_0$, such that for any $\delta\in(0,\delta_0],$ and any $\theta,\tau$ such that $\|\theta-\hat\theta\|\leq \delta$ and $\|\tau-\hat\tau\|\leq\delta$  and a function $h(\delta)\rightarrow0$ as $\delta\rightarrow0$. Then, under similar assumptions as in Assumptions 1 to 5 of \cite{bennett2019deep}, the global solution $\hat\theta_{n}$ is a consistent estimator to the true parameter $\theta_0$, i.e. $\hat\theta_{n}\xrightarrow{p}\theta_0$ when the approximate error $\varepsilon^i<\frac{\eta^i(\epsilon)}{2}$ for every $\epsilon>0$ for each client $i\in[N]$.
        \end{theorem} 
        
The assumptions and the proof of Theorem~\ref{thm:consistency} are included in Appendix~\ref{ap:consistency}.
        
\begin{remark}
Theorem~\ref{thm:consistency} formalizes a tradeoff between data heterogeneity and the consistency of the global estimator in federated learning. If the approximation error $\varepsilon^i$ is large for a client $i\in[N]$, then the solution $\hat\theta_n$ may fail to consistently estimate the true parameter of client $i$. In contrast, when data across clients have similar distribution (i.e., case for low heterogeneity), the federated optimal model $\hat\theta_n$ is consistent across clients.

\end{remark}
 Now, we discuss that the limit points of \fgda will retrieve the local minimax points of the federated optimization problem. 

\subsection{Limit Points of \fgda}
Let $\alpha_1= \frac{\eta}{\gamma}, \alpha_2={\eta}$ be the learning rates for gradient updates to $\theta$ and $\tau$, respectively. For details, refer to Algorithm~\ref{alg:fedgda} in Appendix \ref{sec:algo}. Without loss of generality the \fgda updates are:
\begin{align}
\theta_{t+1}=\theta_{t}-\eta\frac{1}{\gamma}\frac{1}{N}\sum_{i\in[N]}\sum_{r=1}^{R}\nabla_{\theta}U_{\tilde\theta}^i(\theta_{t,r}^i, \tau_{t,r}^i)\text{  and }
\tau_{t+1}=\tau_{t}+\eta\frac{1}{N}\sum_{i\in[N]}\sum_{r=1}^{R}\nabla_{\tau}U_{\tilde\theta}^i(\theta_{t,r}^i, \tau_{t,r}^i)
\end{align}
We call it $\gamma$-\fgda, where $\gamma$ is the ratio of $\alpha_1$ to $\alpha_2$. As $\eta\rightarrow 0$ corresponds to \fgda-flow,
under the smoothness of $U_{\tilde\theta}^i$, Assumption~\ref{asm:gradientheterogeneity} and for some fixed $R$, \fgda-flow becomes:
\begin{align}
    \frac{d\theta}{dt} = - \frac{1}{\gamma} {R}\nabla_{\theta} U_{\tilde\theta}(\theta, \tau)+ \gO\left(\frac{R}{\gamma}  \zeta_{\theta}\right), \text{  and }
    \frac{d\tau}{dt} =  R \nabla_{\tau} U_{\tilde\theta}(\theta, \tau)+ \gO(R\zeta_{\tau}).
\end{align}
We further elaborate on \fgda-flow in Appendix~\ref{ap:fgdaflow}.
\begin{proposition}\label{prop:stricteqm}
    Given the Jacobian matrix for $\gamma-$\fgda flow as
    \begin{equation*}
        \mJ=\begin{pmatrix}
            {-\frac{1}{\gamma}R\nabla_{\theta\theta}^2U_{\tilde\theta}(\theta, \tau)}&{-\frac{1}{\gamma}R\nabla_{\theta\tau}^2U_{\tilde\theta}(\theta, \tau)}\\
            {R \nabla_{\tau\theta}^2 U_{\tilde\theta}(\theta, \tau)}& {R \nabla_{\tau\tau}^2 U_{\tilde\theta}(\theta, \tau)}   
        \end{pmatrix},
    \end{equation*}
     a point $(\theta,\tau)$ is a strictly linearly stable equilibrium of the $\gamma-$\fgda flow if and only if the real parts of all eigenvalues of $\mathbf{J}$ are negative, i.e., 
    $\operatorname{Re}(\emLambda_j) < 0 \quad \text{for all } j.$
    \end{proposition}
Proposition \ref{prop:stricteqm} essentially defines a strictly linearly stable equilibrium of the $\gamma-$\fgda flow. The proof follows a strategy similar to \cite{pmlr-v119-jin20e}. 

With that, let $\gamma$-$\gF\gG\gD\gA$ be the set of strictly linearly stable points of the $\gamma$-\fgda flow, $\mathcal{L}\text{oc}\gM\text{inimax}$ be the set of local minimax points of the federated zero-sum game. Define 
\begin{align*}
	\overline{\infty-\mathcal{FGDA}}&:=\limsup_{\gamma\rightarrow\infty}\gamma-\mathcal{FGDA}:=\cap_{\gamma_0>0}\cup_{\gamma>\gamma_0}\gamma-\mathcal{FGDA}, \text{ and }\\
	\underline{\infty-\mathcal{FGDA}}&:=\liminf_{\gamma\rightarrow\infty}\gamma-\mathcal{FGDA}:=\cup_{\gamma_0>0}\cap_{\gamma>\gamma_0}\gamma-\mathcal{FGDA}.
\end{align*}
We now state the theorem that establishes the stable limit points of $\infty$-$\mathcal{FGDA}$ as local minimax points, up to some degenerate cases. This theorem ensures that solutions to a minimax problem obtained using \fgda in the limit $\gamma\rightarrow\infty$ correspond to equilibrium points.

 \begin{theorem}
 \label{thm:inftyfgda}
     Under Assumption~\ref{asm:twicediff}, $\mathcal{L}\text{oc}\gM\text{inimax}\subset\underline{\infty-\mathcal{FGDA}}\subset\overline{\infty-\mathcal{FGDA}}\subset\mathcal{L}\text{oc}\gM\text{inimax}\cup\gA$, where $\gA:=\{(\theta,\tau)|(\theta,\tau)\text{ is stationary and }\nabla_{\tau\tau}^2U_{\tilde\theta}(\theta, \tau)\text{ is degenerate}\}$. Moreover, if the hessian $\nabla_{\tau\tau}^2 U_{\tilde\theta}(\theta, \tau)$ is smooth, then $\gA$ has measure zero in $\Theta\times\gT\subset\mathbb{R}^{d} \times \mathbb{R}^k$.
 \end{theorem}
Essentially, Theorem \ref{thm:inftyfgda} states that the limit points of \fgda are the local minimax solutions, and thereby the equilibrium solution of the federated zero-sum game at the server, up to some degenerate cases with measure 0. The proof of Theorem \ref{thm:inftyfgda} is included in Appendix~\ref{ap:limitpoint}.

Theorems \ref{sec::thm:globallocalequilibrium}, \ref{thm:consistency}, and \ref{thm:inftyfgda} together complete the theoretical foundation of the pipeline in our work. Obtaining the equilibrium solution of the federated zero-sum game at the server via the \fgda limit points, using Theorem \ref{sec::thm:globallocalequilibrium} we get $\mathcal{E}$-approximate federated equilibrium solutions, wherefrom we obtain clients' approximate local minimax. Finally, applying Theorem \ref{thm:consistency} we retrieve the consistent estimators for GMM at the clients.
\section{Experiments}
\label{sec:experiments}
In the experiments, we extend the experimental evaluations of \cite{bennett2019deep} to a federated setting. We discuss this benchmark choice further in Appendix \ref{ap:relatedexp}. More specifically, we evaluate the ability of \fgmm to fit low and high dimensional data to demonstrate that it converges analogous to the centralized algorithm \dgmm.   
Similar to \cite{bennett2019deep}, we assess two scenarios in regards to $\left((X,Y),Z\right)$: 
\begin{enumerate}[label=(\alph*),parsep=0pt,topsep=0pt,leftmargin=*,itemsep=0pt]
	\item \textbf{The instrumental and treatment variables $Z$ and $X$ are both low-dimensional.} In this case, we use 1-dimensional synthetic datasets corresponding to the following functions: 	
	\begin{enumerate*}[parsep=0pt,topsep=0pt,leftmargin=*,itemsep=0pt]
		\item \textbf{Absolute}: \( g_0(x) = |x| \),
		\item \textbf{Step}: \( g_0(x) = 1_{\{x \geq 0\}} \),		
		\item \textbf{Linear}: \( g_0(x) = x \).
	\end{enumerate*}
	
	To generate the synthetic data, similar to \cite{bennett2019deep,lewis2018adversarialgeneralizedmethodmoments} we apply the following generation process:
	\begin{align}\label{en:lddata}
		Y &= g_0(X) + e + \delta&\text{  and }&X = Z^{(1)}+Z^{(2)} + e + \gamma  \\\label{en:lddata1}
		(Z^{(1)},Z^{(2)}) &\sim \text{Uniform}([-3, 3]^2)&\text{  and }& e \sim \mathcal{N}(0, 1), \quad \gamma, \delta \sim \mathcal{N}(0, 0.1)
	\end{align}
	\item \textbf{$Z$ and $X$ are low-dimensional or high-dimensional or both.} First, $Z$ and $X$ are generated as in (\ref{en:lddata},\ref{en:lddata1}). Then for high-dimensional data, we map $Z$ and $X$ to an image using the mapping:
	\[		
		\text{Image}(x) = \text{Dataset}\left(\text{round}\left(\min\left(\max(1.5x+5, 0), 9\right)\right)\right),		
	\]
	where $(\text{round}(\min(\max(1.5x+5, 0), 9))) $ returns an integer between 0 and 9. Essentially, the function Dataset $(.)$ randomly selects an image following its index. We use datasets FEMNIST (Federated Extended MNIST) and CIFAR10 \cite{caldas2018leaf} for images of size $28\times28$ and $3\times32\times32$, respectively. Thus, we have the following cases:
	\begin{enumerate*}[parsep=0pt,topsep=0pt,leftmargin=*,itemsep=0pt]
		\item $\mathbf{Dataset_z}$: $X = X^{\text{low}}, Z = \text{Image}(Z^{\text{low}})$,
		\item $\mathbf{Dataset_x}$: $Z = Z^{\text{low}}, X = \text{Image}(X^{\text{low}})$, and 
		\item $\mathbf{Dataset_{x,z}}$: $Z = \text{Image}(Z^{\text{low}})$, $X = \text{Image}(X^{\text{low}})$,
	\end{enumerate*}
	where $\mathbf{Dataset}$ takes values $\mathbf{FEMNIST}$ and $\mathbf{CIFAR10}$ and the superscript $low$ indicates the values generated using the process in low-dimensional case.  
\end{enumerate}	
\cite{bennett2019deep} used Optimistic Adam (\oadam), a variant of Adam \cite{kingma2014adam} based stochastic gradient descent ascent algorithm \cite{daskalakis2018training}, which applies mirror descent based gradient updates. It guarantees the last iteration convergence of a GAN \cite{goodfellow2014generative} training problem. It is known that a well-tuned \sgd outperforms Adam in over-parametrized settings \cite{wilson2017marginal}, closely resembling our \fgmm implementation, where the size of neural networks often exceeds the data available on the clients. Considering that, we explored the comparative performance of \gda and \sgda against \oadam for a centralized \dgmm implementation. Note that \gda also aligns with the analytical discussion presented in Section (\ref{sec:res}). We then implemented the federated versions of each of these methods and benchmarked them for solving the federated minimax optimization problem for the \fdgmm algorithm.
\begin{figure*}[htbp!]
	\centering
	\begin{subfigure}[t]{0.30\textwidth}
		\centering
		\includegraphics[width=\textwidth]{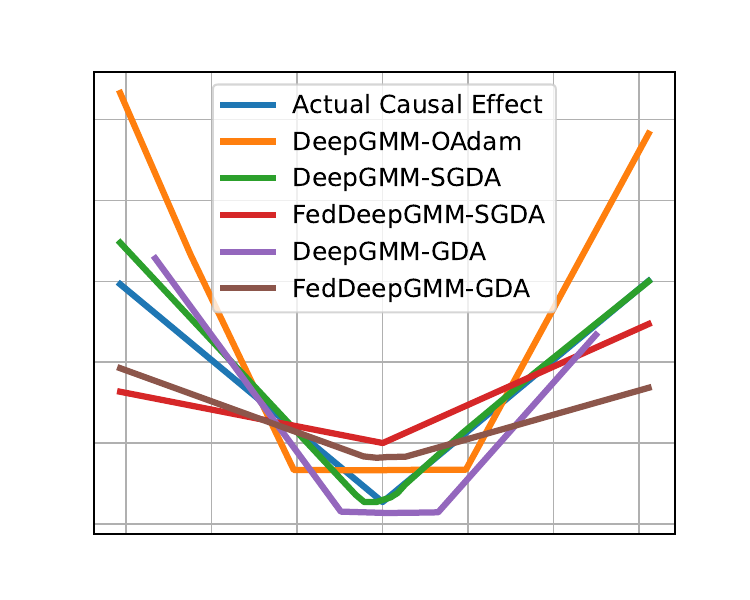 } % Include the image
		\caption{\textbf{Absolute}}		
	\end{subfigure}
	\begin{subfigure}[t]{0.30\textwidth}
	\centering
	\includegraphics[width=\textwidth]{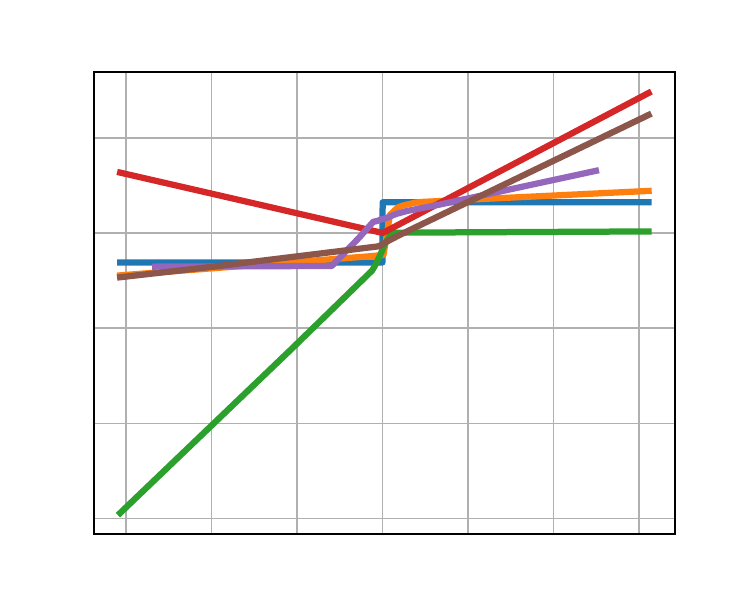}
	\caption{\textbf{Step}}
	\end{subfigure}
	\begin{subfigure}[t]{0.30\textwidth}
		\centering
		\includegraphics[width=\textwidth]{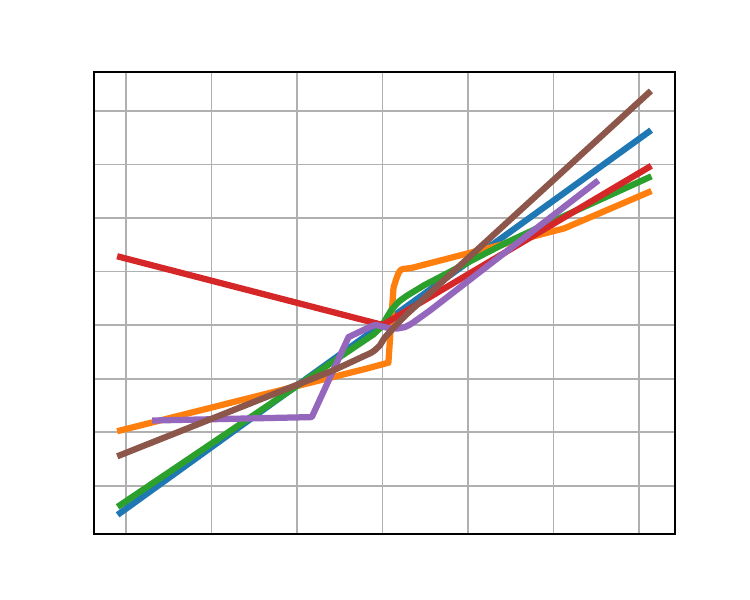 }
		\caption{\textbf{Linear}}
	\end{subfigure}
	\caption{\small Estimated $\hat{g}$ compared to true $g$ in low-dimensional scenarios}\label{exp1}
\end{figure*}
For high-dimensional scenarios, we implement a convolutional neural network (CNN) architecture to process images, while for low-dimensional scenarios, we use a multilayer perceptron (MLP). Code is available at \url{https://github.com/dcll-iiitd/FederatedDeepGMM}.
\begin{table*}[htbp!]
	\small
	\centering
	\begin{tabular}{|p{2cm}|p{1.8cm}|p{1.8cm}|p{1.8cm}|p{2.1cm}|p{1.8cm}|p{2.1cm}|}
		\hline
		\textbf{Estimations} &\textbf{\dgmm-OAdam}&\textbf{\dgmm-GDA} & \textbf{{F}\dgmm-GDA}  & \textbf{\dgmm-SGDA}&\textbf{{F}\dgmm-SGDA}  \\ 
		\hline
		\hline
		\textbf{Absolute}  &0.03 $\pm$ 0.01& 0.013 $\pm$ .01  & 0.4 $\pm$ 0.01 & 0.009 $\pm$ 0.01 & 0.2 $\pm$ 0.00\\
		\hline
		\textbf{Step}  & 0.3 $\pm$ 0.00& 0.03 $\pm$ 0.00  & 0.04 $\pm$ 0.01 & 0.112 $\pm$ 0.00 & 0.23 $\pm$ 0.01\\
		\hline
		\textbf{Linear} & 0.01 $\pm$ 0.00& 0.02 $\pm$ 0.00  & 0.01 $\pm$ 0.00 & 0.03 $\pm$ 0.00& 0.04 $\pm$ 0.00 \\
		\hline
		$\mathbf{FEMNIST_x}$ & 0.50 $\pm$ 0.00 & 1.11 $\pm$ 0.01  & 0.21 $\pm$ 0.02 & 0.40$\pm$ 0.01 & 0.19 $\pm$ 0.01\\
		\hline
		$\mathbf{FEMNIST_{x,z}}$  & 0.24 $\pm$ 0.00& 0.46 $\pm$ 0.09 & 0.19 $\pm$ 0.03 & 0.14$\pm$ 0.02 & 0.20 $\pm$ 0.00\\
		\hline
		$\mathbf{FEMNIST_z}$ & 0.10 $\pm$ 0.00& 0.42 $\pm$ 0.01 & 0.24 $\pm$ 0.01 & 0.11$\pm$ 0.02 & 0.23 $\pm$ 0.01 \\
		\hline%\iffalse
		$\mathbf{CIFAR10_x}$  & 0.55 $\pm$ 0.30  & 0.19 $\pm$ 0.01 & 0.25$\pm$ 0.03 & 0.20 $\pm$ 0.08 &0.22 $\pm$ 0.08\\
		\hline
		$\mathbf{CIFAR10_{x,z}}$   & 0.40 $\pm$ 0.11 & 0.24 $\pm$ 0.00 & 0.24$\pm$ 0.03 & 0.19 $\pm$ 0.03 &0.22 $\pm$ 0.02\\
		\hline
		$\mathbf{CIFAR10_{z}}$ & 0.13 $\pm$ 0.03 & 0.13 $\pm$ 0.01 & 1.70$\pm$ 2.60 & 0.24 $\pm$ 0.01 &0.52 $\pm$ 0.60\\
		\hline%\fi
	\end{tabular}
	\caption{The averaged Test MSE with standard deviation on the low- and high-dimensional scenarios.}
	\label{tab:example_table_1}
\end{table*}\\
\textbf{Non-i.i.d. data.} We sample the train, test and validation sets similar to \cite{bennett2019deep}. For the low-dimensional scenario, we sample $n = 20000$ points for each train, validation, and test set, while, for the high-dimensional scenario, we have $n = 20000$ for the train set and $n=10000$ for the validation and test set. To set up a non-i.i.d. distribution of data between clients, samples were divided amongst the clients using a Dirichlet distribution $Dir_S(\alpha)$ \cite{wang2019federated}, where $\alpha$ determines the degree of heterogeneity across $S$ clients. We used $Dir_S(\alpha) = 0.3$ for each train, test, and validation samples.\\
\textbf{Hyperparameters.} We perform extensive grid-search to tune the learning rate. For \fsgda, we use a minibatch-size of 256. To avoid numerical instability, we standardize the observed $Y$ values by removing the mean and scaling to unit variance. We perform five runs of each experiment and present the mean and standard deviation of the results.\\
\textbf{Observations and Discussion.} In figure (\ref{exp1}), we first observe that \sgda and \gda algorithms perform at par with \oadam to fit the \dgmm estimator. It establishes that hyperparameter tuning is effective. With that, we further observe that the federated algorithms efficiently fit the estimated function to the true data-generating process competitive to the centralized algorithms even though the data is decentralized and non-i.i.d.. Thus, it shows that the federated algorithm converges effectively. In Table~\ref{tab:example_table_1} we present the test mean squared error (MSE) values. The MSE values indicate that the federated implementation achieves competitive convergence to their centralized counterpart. These experiments establish the efficacy of our method.

\section*{An Open Problem}
In this work, we characterized the equilibrium solutions of federated zero-sum games in consideration of local minimax solutions for non-convex non-concave minimax optimization problems. Regardless of the analytical assumptions over the objective, the mixed strategy solutions for zero-sum games exist. However, unlike the pure strategy solutions, where the standard heterogeneity considerations over gradients and Hessians across clients, translates a local minimax solution for the federated objective to approximate local solutions for the clients, it is not immediate how a mixed strategy solution as a probability measure can be translated to that for clients. It leaves an interesting open problem to characterize the mixed startegy solutions for federated zero-sum games.     
\printbibliography

\clearpage
\begin{center}
	\LARGE \textsc{Appendix}
\end{center}
\startcontents[sections]
\printcontents[sections]{l}{1}{\setcounter{tocdepth}{3}}
\appendix
\onecolumn
\section{The Experimental Benchmark Design}
\label{ap:relatedexp}		
It is standard in this area to perform experimental analysis on synthetic datasets for unavailability of ground truth for causal inference; for example see Section 4.1.1 of \citet{vo2022bayesian}. Nonetheless, an experimental comparison of our work with recent works on federated methods for causal effect estimations is not direct. More specifically, see the following:
\begin{enumerate}[label=(\roman*),parsep=0pt,topsep=0pt,leftmargin=*,itemsep=0pt]
	\item \textbf{\textsc{CausalRFF} \cite{vo2022adaptive} and \textsc{FedCI} \cite{vo2022bayesian}.} The aim of \textsc{CausalRFF} \cite{vo2022adaptive} is to estimate the conditional average treatment effect (CATE) and average treatment effect (ATE), whereas \textsc{FedCI} \cite{vo2022bayesian} aims to estimate individual treatment effect (ITE) and ATE. For this, \cite{vo2022adaptive} consider a setting of $Y$ , $W$, and $X$ to be random variables denoting the outcome, treatment, and proxy variable,
	respectively. Along with that, they also consider a confounding variable $Z$. However, their causal dependency builds on the dependence of each of $Y$ , $W$, and $X$ on $Z$ besides dependency of $Y$ on $W$. Consequently, to compute CATE and ATE, they need to estimate the conditional probabilities $p(w^i|x^i)$, $p(y^i|x^i, w^i)$, $p(z^i|x^i, y^i, w^i)$, $p(y^i|w^i, z^i)$, where the superscript $i$ represents a client. Their experiments compare the estimates of CATE and ATE with the Bayesian baselines \cite{hill2011bayesian}, \cite{shalit2017estimating}, \cite{louizos2017causal}, etc. in a centralized setting without any consideration of data decentralization or heterogeneity native to federated learning. Further, they compare against the same baselines in a \textit{one-shot federated} setting, where at the end of training on separate data sources independently, the predicted treatment effects are averaged. Similar is the experimental evaluation of \cite{vo2022bayesian}.
	
	By contrast, the setting of IV analysis as in our work does not consider dependency of the outcome variable $Y$ on the confounder $Z$, though the treatment variable $X$ could be endogenous and depend on $Z$. For us, computing the treatment effects and thereby comparing it against these works is not direct. Furthermore, it is unclear, if the approach of \cite{vo2022adaptive} and \cite{vo2022bayesian}, where the predicted inference over a number of datasets is averaged as the final result, would be comparable to our approach where the problem is solved using a federated maximin optimization with multiple synchronization rounds among the clients. For us, the federated optimization subsumes the experimental of comparing the average predicted values after independent training with the predicted value over the entire data. This is the reason that our centralized counterpart i.e. \dgmm \cite{bennett2019deep}, do not experimentally compare against the baselines of \cite{vo2022adaptive} and \cite{vo2022bayesian}.
	
	In summary, for us the experimental benchmarks were guided by showing the efficient fit of the GMM estimator in a federated setting.
	\item \textbf{\textsc{TEDVAE} \cite{zhang2021treatment}.} As mentioned above, their aim was to showcase the advantage of a weighted averaging over the vanilla averaging of FedAvg. By contrast, our experiments tried to showcase that even in a federated setting, the maximin optimization converges analogous to the centralized counterpart.
\end{enumerate}

\section{Federated Gradient Descent Ascent Algorithm Description}
\label{sec:algo}
\begin{algorithm}[ht]
\caption{\fgda running on a federated learning server to solve the minimax problem (\ref{eqn:fedoptprob})}\label{alg:fedgda}
\textbf{Server Input}: initial global estimate $\theta_1,\tau_1$; constant local learning rate $\alpha_1, \alpha_2$; total $N$ clients\\
\textbf{Output}: global model states $\theta_{T+1}, \tau_{T+1}$
\begin{algorithmic}[1] %[1] enables line numbers
\For{synchronization round $t=1,\ldots,T$}
\State server sends $\theta_t, \tau_t$ to all clients
\For {each $i\in [N]$ in parallel}
\State $\theta_{t,1}^i\leftarrow \theta_t$, $\tau_{t,1}^i\leftarrow \tau_t$
\For{$r=1,2,\ldots,R$}
\State $\theta_{t,r+1}^i= \theta_{t,r}^i-\alpha_1\nabla_{\theta} f_i(\theta_{t,r}^i, \tau_{t,r}^i)$
\State $\tau_{t,r+1}^i= \tau_{t,r}^i+\alpha_2\nabla_{\tau} f_i(\theta_{t,r}^i, \tau_{t,r}^i)$
\EndFor
\State $(\Delta \theta_{t}^i, \Delta \tau_{t})\leftarrow (\theta_{t,R+1}^i-\theta_{t}, \tau_{t,R+1}^i-\tau_{t})$
\EndFor
\State $(\Delta \theta_{t}, \Delta \tau_{t})\leftarrow\frac{1}{N}\sum_{i\in[N]}(\Delta \theta_{t}^i, \Delta \tau_{t}^i)$
\State $\theta_{t+1}\leftarrow (\theta_t + \Delta \theta_{t})$, $\tau_{t+1}\leftarrow (\tau_t + \Delta \tau_{t})$
\EndFor
\State \textbf{return} $\theta_{T+1}; \tau_{T+1}$
\end{algorithmic}
\end{algorithm}

We adapt the proof of Theorem 1 in \cite{sharma2022federated} for the \sgda algorithm proposed in \cite{deng2021local} for the \fgda algorithm~\ref{alg:fedgda} for smooth non-convex- PL problems.

\begin{assumption}[Polyak Łojaisiewicz (PL) condition in $\tau$]
\label{asm:PL}The function $U_{\tilde\theta}$ satisfyies $\mu-PL$ condition in $\tau$, $\mu>0$, if for any fixed $\theta$, $\argmax_{\tau^\prime}U_{\tilde\theta}(\theta,\tau^\prime)\neq \phi$ and $\|\nabla_{\tau}U_{\tilde\theta}(\theta,\tau)\|^2\geq 2\mu\left(\max_{\tau^\prime}U_{\tilde\theta}(\theta,\tau^\prime)-U_{\tilde\theta}(\theta,\tau)\right)$.
    
\end{assumption}

\begin{theorem}
    Let the local loss functions $U_{\tilde\theta}^i$ for all $i\in\{1,2,\dots,N\}$ satisfy assumption~\ref{asm:lipschitz} and~\ref{asm:gradientheterogeneity}. The federated objective function satisfies assumption~\ref{asm:PL}. Suppose $\alpha_2\leq \frac{1}{8LR}$, $\frac{\alpha_1}{\alpha_2}\leq\frac{1}{8\kappa^2},$ where $\kappa=\frac{L}{\mu}$ is the condition number. Let $\bar{\theta}_{T+1}$ is drawn uniformly at random from $\{\theta_t\}_{t=1}^{T+1}$, then the following holds:
    \begin{align*}
        \|\nabla\tilde\Phi(\bar\theta_{T+1})\|^2\leq \gO\left(\kappa^2\left[\frac{\Delta_{\tilde\Phi}}{\alpha_2R(T+1)}\right]\right)+\mathcal{O}\left(\kappa^2(R-1)^2[\alpha_2^2\zeta_\tau^2+\alpha_1^2\zeta_{\theta}^2]\right),
    \end{align*}
    where $\nabla\tilde\Phi(.):=\max_{\tau}U_{\tilde\theta}(.,\tau)$ is the envelope function, $\Delta_{\tilde\Phi}:=\tilde\Phi(\theta_0)-\min_{\theta}\tilde\Phi(\theta),$ and $\zeta_{\theta}:=\frac{1}{N}\sum_{i=1}^{N}\zeta_{\theta}^i,~\zeta_{\tau}:=\frac{1}{N}\sum_{i=1}^{N}\zeta_{\tau}^i$. Using $\alpha_1=\mathcal{O}\left(\frac{1}{\kappa^2}\sqrt{\frac{N}{R(T+1)}}\right)$, $\alpha_2=\mathcal{O}\left(\sqrt{\frac{N}{R(T+1)}}\right)$, $\|\nabla\tilde\Phi(\bar\theta_{T+1})\|^2$ can be bounded as
    \begin{align*}
        \mathcal{O}\left(\frac{\kappa^2\Delta_{\tilde\Phi}}{\sqrt{NR(T+1)}}+\kappa^2(R-1)^2\frac{NR(\zeta_{\theta}^2+\zeta_{\tau}^2)}{R(T+1)}\right).
    \end{align*}
\end{theorem}

Although the original assumption uses the supremum of average squared deviations, say $\zeta_{\theta}^\prime$ and $\zeta_{\tau}^\prime$, we use per-client dissimilarity bounds $ \zeta_\theta^i,~\zeta_\tau^i$ and upper bound their quantity as $ {\zeta_{\theta}^\prime}^2\leq \frac{1}{N} \sum_{i=1}^N (\zeta_\theta^i)^2:= {\zeta_\theta}^2$ and ${\zeta_{\tau}^\prime}^2 \leq \frac{1}{N} \sum_{i=1}^N (\zeta_\tau^i)^2:= {\zeta_\tau}^2$. Since there is no stochasticity, we used the bounded variance $\sigma=0$. For details, refer to proof of Theorem 1 in \cite{sharma2022federated}.

\section{Proofs}
\label{ap:proofs}
\subsection{Proof of Lemma~\ref{lemma:federatedU}}
\label{ap::lemma:federatedU}
\begin{lemma}[Restatement of Lemma \ref{lemma:federatedU}]
	Let $\gF = \text{span}\{f^i_j~\vert~i\in[N],~j\in[m]\}$. An equivalent objective function for the federated moment estimation optimization problem (\ref{eqn:FMC}) is given by:
	\begin{equation}
		\left\|  \psi_{N}(f; \theta)\right\|_{\tilde{\theta}}^2 =\sup_{\substack{f^i\in \mathcal{F}\\ \forall i\in[N]}}\frac{1}{N}\sum_{i=1}^{N} \left(\psi_{n_i}(f^i;\theta)-\frac{1}{4}\gC_{\tilde{\theta}}(f^i;f^i)\right), \text{ where}
	\end{equation}
	\[\begin{aligned}
		&\psi_{n_i}(f^i; \theta) :=  \frac{1}{n_i} \sum_{k=1}^{n_i} f^i(Z_k^i)(Y^i_k - g^i(X^i_k;\theta)), \text{ and }
		\mathcal{C}_{\tilde{\theta}}(f^i, f^i) := \frac{1}{n_i} \sum_{k=1}^{n_i} (f^i(Z^i_k))^2 (Y_k^i - g^i(X_k^i;\tilde{\theta}))^2.
	\end{aligned}
	\]
\end{lemma}

\begin{proof}
Let $\psi=(\frac{1}{N}\sum_{i=1}^{N}\psi_{n_i}(f_1^i;\theta), \frac{1}{N}\sum_{i=1}^{N}\psi_{n_i}(f_2^i;\theta), \dots, \frac{1}{N}\sum_{i=1}^{N}\psi_{n_i}(f_m^i;\theta))$.

We know that $\|v\|^2= v^\top C_{\tilde{\theta}}^{-1}v$ and the associated dual norm is obtained as $\|v\|_*= \sup_{\|v\|\leq 1} v^\top v= v^\top C_{\tilde{\theta}}v$.

Using the definition of the dual norm,
    \begin{align}
        \|\psi\|&=  \sup_{\|v\|_*\leq 1} v^\top \psi\nonumber\\
        \|\psi\|^2&=  \sup_{\|v\|_*\leq \|\psi\|} v^\top \psi\nonumber\\
        \|\psi\|^2&=  \sup_{v^\top C_{\tilde{\theta}}v \leq \|\psi\|^2} v^\top \psi\label{eqn:primal-max}.
    \end{align}
    We now find the equivalent dual optimization problem for~(\ref{eqn:primal-max}).
    
    The Lagrangian of the constrained maximization problem~(\ref{eqn:primal-max}) is given as
    \begin{equation*}
        \gL(v,\lambda)= v^\top \psi+\lambda(v^\top C_{\tilde{\theta}}v - \|\psi\|^2),~\text{where}~\lambda \le 0.
    \end{equation*} 
    To maximize $ \gL(v,\lambda)$ w.r.t. $v$, put $\frac{   \partial\gL}{\partial v}= \psi+2\lambda C_{\tilde{\theta}}v=0$ to obtain $v=\frac{-1}{2\lambda}C_{\tilde{\theta}}^{-1}\psi.$

    When $\|\psi\|> 0$, $v=0$ satisfies the Slater's condition as a strictly feasible interior point of the constraint $v^\top C_{\tilde{\theta}}v - \|\psi\|^2\leq 0$. Thus, strong duality holds.
    Substituting $v=\frac{-1}{2\lambda}C_{\tilde{\theta}}^{-1}\psi$ in the Lagrangian gives
     \begin{align*}
        \gL^*(\lambda)&= \frac{-1}{2\lambda}{\psi}^\top C_{\tilde{\theta}}^{-1}\psi+\frac{1}{4\lambda}{\psi}^\top C_{\tilde{\theta}}^{-1}\psi - \lambda\|\psi\|^2\\
       &= -\frac{\|\psi\|^2}{4\lambda} - \lambda\|\psi\|^2 .
    \end{align*}
    Hence, the dual becomes $\|\psi\|^2=inf_{\lambda <0}\left\{ \gL^*(\lambda)\right\}$. Thus, the equivalent dual optimization problem for~(\ref{eqn:primal-max}) is given as
    \begin{equation}
        \|\psi\|^2=\inf_{\lambda <0}\left\{-\frac{\|\psi\|^2}{4\lambda} - \lambda\|\psi\|^2 \right\}.
    \end{equation}

    Putting $\frac{   \partial\gL}{\partial \lambda}= \frac{\|\psi\|^2}{4\lambda^2} -\|\psi\|^2=0$ gives $\lambda=\frac{-1}{2}.$ Thus, due to strong duality $\|\psi\|^2= \sup_{v} \gL(v,\frac{-1}{2})= \sup_{v} v^\top \psi-\frac{1}{2}(v^\top C_{\tilde{\theta}}v - \|\psi\|^2).$ 
    
    Rewriting it $\frac{1}{2}\|\psi\|^2= \sup_{v}   v^\top \psi-\frac{1}{2}v^\top C_{\tilde{\theta}}v$ and substituting $u=2v$
    \begin{equation*}
        \|\psi\|^2= \sup_{u}   u^\top \psi-\frac{1}{4}u^\top C_{\tilde{\theta}}u .
    \end{equation*}
    Using change of variables $u\rightarrow v$
       \begin{equation*}
        \|\psi\|^2= \sup_{v}   v^\top \psi-\frac{1}{4}v^\top C_{\tilde{\theta}}v.
    \end{equation*}
    Now, we want to find a function form for the optimization problem mentioned above.

    Consider a finite-dimensional functional spaces $\gF^i=\text{span}\{f_1^i, f_2^i, \dots, f_m^i\}$ for each client $i$. Hence, for $f^i\in\gF^i$ 
    \begin{equation*}
        f^i=\sum_{j=1}^{m}v_jf^i_j.
    \end{equation*}
Since all the clients share the same neural network architecture, we define a global functional space $\gF$ as 
\begin{equation*}
    \gF=\text{span}\{f^i_j~\vert~i\in[N],~j\in[m]\}.
\end{equation*}
Therefore, $v$ corresponds to $f^i$ such that \begin{equation*}
    f^i=\sum_{c=1}^{N}\sum_{j=1}^{m}v^i_j f^c_j,
\text{ where }
    v^i_j=
 \begin{cases}
        v_j & \text{if } c =i\\
        0 & \text{if } c \neq i
    \end{cases}
    \end{equation*}
Hence, 
\begin{align*}
v^\top\psi&=\frac{1}{N}\sum_{i=1}^{N}\sum_{j=1}^{m}v_j\psi_{n_i}(f^i_j;\theta)\\
&=\frac{1}{N}\sum_{i=1}^{N}\frac{1}{n_i}\sum_{k=1}^{n_i}f^i(Z_k^i)(Y^i_k-g^i(X^i_k;\theta)).
\end{align*} 
Similarly, 
\begin{align*}
v^\top C_{\tilde{\theta}}v&=\sum_{p=1}^{m}\sum_{q=1}^{m}v_pv_q[C_{\tilde{\theta}}]pq\\
&=\sum_{p=1}^{m}\sum_{q=1}^{m}v_pv_q\frac{1}{N}\sum_{i=1}^{N}\frac{1}{n_i}\sum_{k=1}^{n_i}f^i_p(Z^i_k)f_q^i(Z_k^i)(Y_k^i-g^i(X_k^i;\tilde\theta))\\
&=\frac{1}{N}\sum_{i=1}^{N}\frac{1}{n_i}\sum_{k=1}^{n_i}\sum_{p=1}^{m}v_pf^i_p(Z^i_k)\sum_{q=1}^{m}v_q f_q^i(Z_k^i)(Y_k^i-g^i(X_k^i;\tilde\theta))^2\\
&=\frac{1}{N}\sum_{i=1}^{N}\frac{1}{n_i}\sum_{k=1}^{n_i}(f^i(Z^i_k))^2(Y_k^i-g^i(X_k^i;\tilde\theta))^2\\
&=\frac{1}{N}\sum_{i=1}^{N}\gC_{\tilde\theta}(f^i,f^i).
\end{align*}
Thus, applying the Riesz Representation theorem using the representations $v^\top\psi=\frac{1}{N}\sum_{i=1}^{N}\psi_{n_i}(f^i;\theta)$ and $v^\top C_{\tilde{\theta}}v=\frac{1}{N}\sum_{i=1}^{N}\gC_{\tilde{\theta}}(f^i,f^i)$, we can write the objective in functional form as
      \begin{equation*}
    \|\psi\|^2 = 
    \sup_{\substack{f^i\in \mathcal{F}\\ \forall i\in[N]}}
       \frac{1}{N} \sum_{i=1}^N\left( 
  \psi_{n_i}(f^i; \theta) 
    - \frac{1}{4}\mathcal{C}_{\tilde{\theta}}(f^i, f^i) 
    \right).
    \end{equation*}
This gives us the desired result.

\end{proof}

\subsection{Proof of Theorem~\ref{sec::thm:globallocalequilibrium}}
\label{ap::thm:globallocalequilibrium}
\begin{theorem}[Restatement of Theorem~\ref{sec::thm:globallocalequilibrium}]
	Under assumptions~\ref{asm:twicediff},~\ref{asm:lipschitz},~\ref{asm:gradientheterogeneity} and~\ref{asm:hessianheterogeneity}, a minimax solution $(\hat\theta,\hat\tau)$ of federated optimization problem~(\ref{eqn:fedoptprob}) that satisfies the equilibrium condition as in definition~\ref{def:localequilibrium}: 
	%          \begin{equation*}
		$
		U_{\tilde\theta}(\hat\theta,\tau)\leq U_{\tilde\theta}(\hat\theta,\hat\tau)\leq \max_{{\tau\prime : \|\tau\prime-\hat\tau\|\leq h(\delta)}}U_{\tilde\theta}(\theta,\tau\prime),
		$
		%     \end{equation*}
	is an $\mathcal{E}$-approximate federated equilibrium solution as defined in \ref{def:epsilonequilibrium}, where the approximation error $\varepsilon^i$ for each client $i\in[N]$ lies in:
	%\begin{equation*}
	$ \max\{ \zeta_{\theta}^i, \zeta_{\tau}^i\}\le \varepsilon^i\le \min\{\alpha-\rho_{\tau}^i,\beta- B^i \}    
	$
	%\end{equation*} 
	for $\rho_{\tau}^i<\alpha$ and $B^i>\beta$, such that $\alpha:=\left\vert\lambda_{\max}\left(\nabla_{\tau\tau}^2U_{\tilde\theta}(\hat\theta,\hat\tau)\right)\right\vert$, $\beta:=\lambda_{\min}\left(\left[\nabla_{\theta\theta}^2U_{\tilde\theta}-\nabla_{\theta\tau}^2U_{\tilde\theta}\left(\nabla_{\tau\tau}^2U_{\tilde\theta}\right)^{-1}\nabla_{\tau\theta}^2U_{\tilde\theta}\right](\hat\theta,\hat\tau)\right)$ and $B^i:=\rho_{\theta}^i+ L\rho_{\theta\tau}^i\frac{1}{\vert\lambda_{\max}(\nabla_{\tau\tau}^2 U_{\tilde\theta}^i)\vert}+L\rho_{\tau\theta}^i\frac{1}{\vert\lambda_{\max}(\nabla_{\tau\tau}^2 U_{\tilde\theta}^i)\vert}+ L^2\rho_{\tau}^i\frac{1}{\vert\lambda_{\max}(\nabla_{\tau\tau}^2 U_{\tilde\theta}^i)\cdot\lambda_{\max}(\nabla_{\tau\tau}^2 U_{\tilde\theta})\vert}$.
\end{theorem}

\begin{proof}
    The pure-strategy Stackelberg equilibrium for the federated objective is:
\begin{equation}
\label{eqn:equilibrium}
         U_{\tilde\theta}(\hat\theta,\tau)\leq U_{\tilde\theta}(\hat\theta,\hat\tau)\leq \max_{{\tau\prime : \|\tau\prime-\tau^*\|\leq h(\delta)}}U_{\tilde\theta}(\theta,\tau\prime),
     \end{equation}
We want to show that the $\epsilon^i$- approximate equilibrium for each client's objective $U_{\tilde{\theta}}^i $ also hold individually.

The first-order necessary condition for (\ref{eqn:equilibrium}) to hold is $\nabla_{\theta}U_{\tilde\theta}(\hat\theta,\hat\tau)=0$ and $\nabla_{\tau}U_{\tilde\theta}(\hat\theta,\hat\tau)=0$. Thus, $\left\|\nabla_{\theta}U_{\tilde\theta}(\hat\theta,\hat\tau)\right\|^2=0$.

Consider
\begin{align*}
\left\|\nabla_{\theta}U_{\tilde\theta}(\hat\theta,\hat\tau)\right\|^2&=
\left\|\nabla_{\theta}U_{\tilde\theta}(\hat\theta,\hat\tau)-\nabla_{\theta}U_{\tilde\theta}^i(\hat\theta,\hat\tau)+\nabla_{\theta}U_{\tilde\theta}^i(\hat\theta,\hat\tau)\right\|^2\\
&=\left\|\nabla_{\theta}U_{\tilde\theta}(\hat\theta,\hat\tau)-\nabla_{\theta}U_{\tilde\theta}^i(\hat\theta,\hat\tau)\right\|^2+\left\|\nabla_{\theta}U_{\tilde\theta}^i(\hat\theta,\hat\tau)\right\|^2\\
&+2\left(\nabla_{\theta}U_{\tilde\theta}(\hat\theta,\hat\tau)-\nabla_{\theta}U_{\tilde\theta}^i(\hat\theta,\hat\tau)\right)^\top\left(\nabla_{\theta}U_{\tilde\theta}^i(\hat\theta,\hat\tau)\right)
\end{align*}
Rearranging
\begin{align*}
 2\left(\nabla_{\theta}U_{\tilde\theta}^i(\hat\theta,\hat\tau)-\nabla_{\theta}U_{\tilde\theta}(\hat\theta,\hat\tau)\right)^\top\left(\nabla_{\theta}U_{\tilde\theta}^i(\hat\theta,\hat\tau)\right)-\left\|\nabla_{\theta}U_{\tilde\theta}^i(\hat\theta,\hat\tau)\right\|^2&= \left\|\nabla_{\theta}U_{\tilde\theta}(\hat\theta,\hat\tau)-\nabla_{\theta}U_{\tilde\theta}^i(\hat\theta,\hat\tau)\right\|^2\\
\left\|\nabla_{\theta}U_{\tilde\theta}^i(\hat\theta,\hat\tau)\right\|^2-2\left(\nabla_{\theta}U_{\tilde\theta}(\hat\theta,\hat\tau)\right)^\top\left(\nabla_{\theta}U_{\tilde\theta}^i(\hat\theta,\hat\tau)\right)&=  \left\|\nabla_{\theta}U_{\tilde\theta}(\hat\theta,\hat\tau)-\nabla_{\theta}U_{\tilde\theta}^i(\hat\theta,\hat\tau)\right\|^2
\end{align*}
Using gradient heterogeneity assumption~(\ref{asm:gradientheterogeneity}) on R.H.S.
\begin{equation*}
 \left\|\nabla_{\theta}U_{\tilde\theta}(\hat\theta,\hat\tau)-\nabla_{\theta}U_{\tilde\theta}^i(\hat\theta,\hat\tau)\right\|^2\leq (\zeta_{\theta}^i)^2
 \end{equation*}
Thus, we obtain $\left\|\nabla_{\theta}U_{\tilde\theta}^i(\hat\theta,\hat\tau)\right\| \leq \zeta_{\theta}^i.$ Similarly, $\left\|\nabla_{\tau}U_{\tilde\theta}^i(\hat\theta,\hat\tau)\right\| \leq \zeta_{\tau}^i$.

In the special case, when $\zeta_{\theta}^i=0$ and $\zeta_{\tau}^i=0$, thus we will have $\left\|\nabla_{\theta}U_{\tilde\theta}^i(\hat\theta,\hat\tau)\right\|^2=\left\|\nabla_{\tau}U_{\tilde\theta}^i(\hat\theta,\hat\tau)\right\|^2=0$ for all $i\in [N]$, which gives $\nabla_{\theta}U_{\tilde\theta}^i(\hat\theta,\hat\tau)=\nabla_{\tau}U_{\tilde\theta}^i(\hat\theta,\hat\tau)=0$ for all clients $i$.

Next, we prove that each client satisfies the second-order necessary condition approximately.
Since $(\hat\theta,\hat\tau)$ satisfy the equilibrium condition~(\ref{eqn:equilibrium}), the second-order necessary condition holds for the global function $U_{\tilde\theta}$, i.e. $\nabla_{\tau\tau}^2U_{\tilde\theta}(\hat\theta,\hat\tau)\preceq \mathbf{0}.$ We now prove that $\nabla_{\tau\tau}^2U_{\tilde\theta}^i(\hat\theta,\hat\tau)\preceq \mathbf{0}.$
% If $\nabla_{\tau\tau}^2U_{\tilde\theta}(\theta,\tau)\prec~\mathbf{0}$, then $ \left[\nabla_{\theta\theta}^2U_{\tilde\theta}-\nabla_{\theta\tau}^2U_{\tilde\theta}\left(\nabla_{\tau\tau}^2U_{\tilde\theta}\right)^{-1}\nabla_{\tau\theta}^2U_{\tilde\theta}\right](\theta,\tau)\succeq \mathbf{0}$.

Using assumption~\ref{asm:twicediff}, the hessian is symmetric. Thus, $\nabla_{\tau\tau}^2U_{\tilde\theta}(\hat\theta,\hat\tau)\preceq \mathbf{0}$ implies $\lambda_{\max}(\nabla_{\tau\tau}^2U_{\tilde\theta}(\hat\theta,\hat\tau))\leq 0$, where $\lambda_{\max}$ is the largest eigenvalue of the hessian. Suppose,  $\lambda_{\max}(\nabla_{\tau\tau}^2U_{\tilde\theta}(\hat\theta,\hat\tau))=-\alpha$, for some $\alpha\geq 0$.

We can write $\nabla_{\tau\tau}^2U_{\tilde\theta}^i(\hat\theta,\hat\tau)=\nabla_{\tau\tau}^2U_{\tilde\theta}^i(\hat\theta,\hat\tau)-\nabla_{\tau\tau}^2U_{\tilde\theta}(\hat\theta,\hat\tau)+\nabla_{\tau\tau}^2U_{\tilde\theta}(\hat\theta,\hat\tau)$.

Using a corollary of Weyl's theorem \cite{Horn_Johnson_2012} for real symmetric matrices $A$ and $B$, $ \lambda_{\max}(A+B)\leq \lambda_{\max}(A)+\lambda_{\max}(B)$. Hence,
\begin{equation*}
  \lambda_{\max}(\nabla_{\tau\tau}^2U_{\tilde\theta}^i(\hat\theta,\hat\tau))\leq \lambda_{\max}(\nabla_{\tau\tau}^2U_{\tilde\theta}^i(\hat\theta,\hat\tau)-\nabla_{\tau\tau}^2U_{\tilde\theta}(\hat\theta,\hat\tau))+\lambda_{\max}(\nabla_{\tau\tau}^2U_{\tilde\theta}(\hat\theta,\hat\tau)).
\end{equation*}
Thus, $\lambda_{\max}(\nabla_{\tau\tau}^2U_{\tilde\theta}^i(\hat\theta,\hat\tau))\leq \lambda_{\max}(\nabla_{\tau\tau}^2U_{\tilde\theta}^i(\hat\theta,\hat\tau)-\nabla_{\tau\tau}^2U_{\tilde\theta}(\hat\theta,\hat\tau))-\alpha$.

Since the spectral norm of a real symmetric matrix A is given as $\lVert A\rVert_{\sigma}=\max\{\vert\lambda_{\max}(A)\vert, \vert\lambda_{\min}(A)\vert\}$.

Under hessian heterogeneity assumption~\ref{asm:hessianheterogeneity}
\begin{align*}
\|\nabla_{\tau\tau}^2U_{\tilde\theta}^i(\hat\theta,\hat\tau)-\nabla_{\tau\tau}^2U_{\tilde\theta}(\hat\theta,\hat\tau)\|_{\sigma}
&=\max\left\{\left\vert\lambda_{\max}(\nabla_{\tau\tau}^2U_{\tilde\theta}^i(\theta,\tau)-\nabla_{\tau\tau}^2U_{\tilde\theta}(\theta,\tau))\right\vert,\right.\\ &\left.\left\vert\lambda_{\min}(\nabla_{\tau\tau}^2U_{\tilde\theta}^i(\theta,\tau)-\nabla_{\tau\tau}^2U_{\tilde\theta}(\theta,\tau))\right\vert\right\}\\
&\leq \rho_{\tau}^i.
\end{align*}

By definition of the spectral norm $\|\nabla_{\tau\tau}^2U_{\tilde\theta}^i(\hat\theta,\hat\tau)-\nabla_{\tau\tau}^2U_{\tilde\theta}(\hat\theta,\hat\tau)\|_{\sigma}=\lambda_{max}(\nabla_{\tau\tau}^2U_{\tilde\theta}^i(\hat\theta,\hat\tau)-\nabla_{\tau\tau}^2U_{\tilde\theta}(\hat\theta,\hat\tau))$,
\begin{align*}
    \lambda_{max}(\nabla_{\tau\tau}^2U_{\tilde\theta}^i(\hat\theta,\hat\tau)-\nabla_{\tau\tau}^2U_{\tilde\theta}(\hat\theta,\hat\tau))
    &\leq\max\left\{\left\vert\lambda_{\max}(\nabla_{\tau\tau}^2U_{\tilde\theta}^i(\hat\theta,\hat\tau)-\nabla_{\tau\tau}^2U_{\tilde\theta}(\hat\theta,\hat\tau))\right\vert,\right.\\ &\left.\left\vert\lambda_{\min}(\nabla_{\tau\tau}^2U_{\tilde\theta}^i(\hat\theta,\hat\tau)-\nabla_{\tau\tau}^2U_{\tilde\theta}(\hat\theta,\hat\tau))\right\vert\right\}\\
&\leq \rho_{\tau}^i.
\end{align*}

Thus, $\lambda_{\max}(\nabla_{\tau\tau}^2U_{\tilde\theta}^i(\hat\theta,\hat\tau))\leq \lambda_{\max}(\nabla_{\tau\tau}^2U_{\tilde\theta}^i(\hat\theta,\hat\tau)-\nabla_{\tau\tau}^2U_{\tilde\theta}(\hat\theta,\hat\tau))-\alpha\leq \rho_{\tau}^i-\alpha$, where $\rho_{\tau}^i\geq 0$.
Hence, \[\nabla_{\tau\tau}^2U_{\tilde\theta}^i(\hat\theta,\hat\tau)\preceq (\rho_{\tau}^i-\alpha)\mathbf{I}.\]
When $\rho_{\tau}^i\le\alpha$, then $\nabla_{\tau\tau}^2U_{\tilde\theta}^i(\hat\theta,\hat\tau)\preceq 0$.

\noindent Now, since $(\hat\theta,\hat\tau)$ satisfy the equilibrium condition~(\ref{eqn:equilibrium}), thus $\nabla_{\tau\tau}^2U_{\tilde\theta}(\hat\theta,\hat\tau)\prec 0$ and the Schur complement of $\nabla_{\tau\tau}^2U_{\tilde\theta}(\hat\theta,\hat\tau)$ is positive semi-definite. 
Now when $\rho_{\tau}^i<\alpha$, it follows from above that $\nabla_{\tau\tau}^2U_{\tilde\theta}^i(\hat\theta,\hat\tau)\prec 0$, hence $\left(\nabla_{\tau\tau}^2U_{\tilde\theta}^i(\hat\theta,\hat\tau)\right)^{-1}$ exists.
Now, we need to show that Schur complement of $\nabla_{\tau\tau}^2U_{\tilde\theta}^i(\hat\theta,\hat\tau)$ is positive semi-definite.

\noindent Since,
$ S(\hat\theta,\hat\tau):=\left[\nabla_{\theta\theta}^2U_{\tilde\theta}-\nabla_{\theta\tau}^2U_{\tilde\theta}\left(\nabla_{\tau\tau}^2U_{\tilde\theta}\right)^{-1}\nabla_{\tau\theta}^2U_{\tilde\theta}\right](\hat\theta,\hat\tau)\succ 0.$ 

Define $S^i:=\left[\nabla_{\theta\theta}^2U_{\tilde\theta}^i-\nabla_{\theta\tau}^2U_{\tilde\theta}^i\left(\nabla_{\tau\tau}^2U_{\tilde\theta}^i\right)^{-1}\nabla_{\tau\theta}^2U_{\tilde\theta}^i\right]$. We aim to prove $\lambda_{\min}(S^i)\geq0$ to show $S^i$ is positive semidefinite (PSD).

\noindent Analogous to the above part, using corollary to Weyl's theorem, we have
\begin{equation*}
\lambda_{\min}(S^i-S)+\lambda_{\min}(S)\leq   \lambda_{\min}(S^i).
\end{equation*}
Let $\lambda_{\min}(S)=\beta$, where $\beta\ge 0$. Moreover, $\|S^i-S\|_{\sigma}=\max\left\{\left\vert\lambda_{\max}(S^i-S)\right\vert, \left\vert\lambda_{\min}(S^i-S)\right\vert\right\}$, thus $\lambda_{\min}(S^i-S)\geq -\|S^i-S\|_{\sigma}$.

\noindent Thus, we have
\begin{equation*}
-\|(S^i-S)\|_{\sigma}+\beta\leq   \lambda_{\min}(S^i).
\end{equation*}
We can write $S^i-S$ as
   \begin{align*}
        S^i - S &= (\nabla_{\theta\theta}^2 U_{\tilde\theta}^i - \nabla_{\theta\theta}^2 U_{\tilde\theta})- \Big[ (\nabla_{\theta\tau}^2 U_{\tilde\theta}^i - \nabla_{\theta\tau}^2 U_{\tilde\theta}) (\nabla_{\tau\tau}^2 U_{\tilde\theta}^i)^{-1} \nabla_{\tau\theta}^2 U_{\tilde\theta}^i \\
        &\quad + \nabla_{\theta\tau}^2 U_{\tilde\theta} (\nabla_{\tau\tau}^2 U_{\tilde\theta}^i)^{-1} (\nabla_{\tau\theta}^2 U_{\tilde\theta}^i - \nabla_{\tau\theta}^2 U_{\tilde\theta}) + \nabla_{\theta\tau}^2 U_{\tilde\theta} \Big( (\nabla_{\tau\tau}^2 U_{\tilde\theta}^i)^{-1} - (\nabla_{\tau\tau}^2 U_{\tilde\theta})^{-1} \Big) \nabla_{\tau\theta}^2 U_{\tilde\theta} \Big].
    \end{align*}
Hence,
   \begin{align*}
        \|S^i - S\|_{\sigma} &\leq \|\nabla_{\theta\theta}^2 U_{\tilde\theta}^i - \nabla_{\theta\theta}^2 U_{\tilde\theta}\|_{\sigma}+\underbrace{\|(\nabla_{\theta\tau}^2 U_{\tilde\theta}^i - \nabla_{\theta\tau}^2 U_{\tilde\theta}) (\nabla_{\tau\tau}^2 U_{\tilde\theta}^i)^{-1} \nabla_{\tau\theta}^2 U_{\tilde\theta}^i\|_{\sigma}}_{\text{$T_1$}} \\
        &\quad +\underbrace{\|\nabla_{\theta\tau}^2 U_{\tilde\theta} (\nabla_{\tau\tau}^2 U_{\tilde\theta}^i)^{-1} (\nabla_{\tau\theta}^2 U_{\tilde\theta}^i - \nabla_{\tau\theta}^2 U_{\tilde\theta})\|_{\sigma}}_{\text{$T_2$}} \\
        &\quad + \underbrace{\|\nabla_{\theta\tau}^2 U_{\tilde\theta} \Big( (\nabla_{\tau\tau}^2 U_{\tilde\theta}^i)^{-1} - (\nabla_{\tau\tau}^2 U_{\tilde\theta})^{-1} \Big) \nabla_{\tau\theta}^2 U_{\tilde\theta} \|_{\sigma}}_{\text{$T_3$}}.
    \end{align*}
Note that the eigenvalue of $(\nabla_{\tau\tau}^2 U_{\tilde\theta}^i)^{-1}$is $\lambda\left((\nabla_{\tau\tau}^2 U_{\tilde\theta}^i)^{-1}\right)=\frac{1}{\lambda\left(\nabla_{\tau\tau}^2 U_{\tilde\theta}^i\right)}$, hence $\|(\nabla_{\tau\tau}^2 U_{\tilde\theta}^i)^{-1}\|_{\sigma}=\frac{1}{\vert\lambda_{\max}(\nabla_{\tau\tau}^2 U_{\tilde\theta}^i)\vert}$ as $\nabla_{\tau\tau}^2 U_{\tilde\theta}^i$ is negative definite.
By Assumption~\ref{asm:lipschitz}, each client's function $U^i$ is $L$-Lipschitz  thus $\|\nabla^2 U_{\tilde\theta}^i\|_{\sigma}\leq L$. 
Since the Hessian \( \nabla^2 U_{\tilde\theta}^i \) is a block matrix of the form:
\begin{align*} 
    \nabla^2 U_{\tilde\theta}^i =
\begin{bmatrix}
\nabla_{\theta\theta}^2 U_{\tilde\theta}^i & \nabla_{\theta\tau}^2 U_{\tilde\theta}^i \\
\nabla_{\tau\theta}^2 U_{\tilde\theta}^i & \nabla_{\tau\tau}^2 U_{\tilde\theta}^i
\end{bmatrix},
\end{align*}
The norm of Hessian is at least the norm of one of its components
\begin{equation*}
\|\nabla_{\theta\theta}^2 U_{\tilde\theta}^i\|_{\sigma} \leq L, \quad
\|\nabla_{\theta\tau}^2 U_{\tilde\theta}^i\|_{\sigma} \leq L, \quad
\|\nabla_{\tau\theta}^2 U_{\tilde\theta}^i\|_{\sigma} \leq L, \quad
\|\nabla_{\tau\tau}^2 U_{\tilde\theta}^i\|_{\sigma} \leq L.
\end{equation*}
\noindent Thus, each Hessian block is individually bounded by $L$.
Additionally, $U$ is $L$-Lipschitz too.
\noindent Using Assumption~\ref{asm:hessianheterogeneity}, bounding $T_1$
\begin{align*}
    T_1&=\|(\nabla_{\theta\tau}^2 U_{\tilde\theta}^i - \nabla_{\theta\tau}^2 U_{\tilde\theta}) (\nabla_{\tau\tau}^2 U_{\tilde\theta}^i)^{-1} \nabla_{\tau\theta}^2 U_{\tilde\theta}^i\|_{\sigma}\\
    &\leq \|(\nabla_{\theta\tau}^2 U_{\tilde\theta}^i - \nabla_{\theta\tau}^2 U_{\tilde\theta}) \|_{\sigma}\cdot\|(\nabla_{\tau\tau}^2 U_{\tilde\theta}^i)^{-1}\|_{\sigma}\cdot\|\nabla_{\tau\theta}^2 U_{\tilde\theta}^i\|_{\sigma}\\
    &\leq L\rho_{\theta\tau}^i\frac{1}{\vert\lambda_{\max}(\nabla_{\tau\tau}^2 U_{\tilde\theta}^i)\vert}
\end{align*}
Similarly, bounding $T_2$
\begin{align*}
    T_2&=\|\nabla_{\theta\tau}^2 U_{\tilde\theta} (\nabla_{\tau\tau}^2 U_{\tilde\theta}^i)^{-1} (\nabla_{\tau\theta}^2 U_{\tilde\theta}^i - \nabla_{\tau\theta}^2 U_{\tilde\theta})\|_{\sigma}\\
    &\leq \|\nabla_{\theta\tau}^2 U_{\tilde\theta}\|_{\sigma}\cdot \|(\nabla_{\tau\tau}^2 U_{\tilde\theta}^i)^{-1}\|_{\sigma}\cdot\|(\nabla_{\tau\theta}^2 U_{\tilde\theta}^i - \nabla_{\tau\theta}^2 U_{\tilde\theta})\|_{\sigma}\\
    &\leq L\rho_{\tau\theta}^i\frac{1}{\vert\lambda_{\max}(\nabla_{\tau\tau}^2 U_{\tilde\theta}^i)\vert}
\end{align*}
Lastly we bound $T_3$, it is easy to verify that $\mA^{-1}-\mB^{-1}=\mA^{-1}\left(\mB-\mA\right)\mB^{-1}$
\begin{align*}
    T_3&=\|\nabla_{\theta\tau}^2 U_{\tilde\theta} \Big( (\nabla_{\tau\tau}^2 U_{\tilde\theta}^i)^{-1} - (\nabla_{\tau\tau}^2 U_{\tilde\theta})^{-1} \Big) \nabla_{\tau\theta}^2 U_{\tilde\theta} \|_{\sigma}\\
    &\leq \|\nabla_{\theta\tau}^2 U_{\tilde\theta}\|_{\sigma}\cdot\| (\nabla_{\tau\tau}^2 U_{\tilde\theta}^i)^{-1} - (\nabla_{\tau\tau}^2 U_{\tilde\theta})^{-1} \|_{\sigma}\cdot\|\nabla_{\tau\theta}^2 U_{\tilde\theta}\|_{\sigma}\\
     &= \|\nabla_{\theta\tau}^2 U_{\tilde\theta}\|_{\sigma}\cdot\| (\nabla_{\tau\tau}^2 U_{\tilde\theta}^i)^{-1} (\nabla_{\tau\tau}^2 U_{\tilde\theta}-\nabla_{\tau\tau}^2 U_{\tilde\theta}^i) (\nabla_{\tau\tau}^2 U_{\tilde\theta})^{-1} \|_{\sigma}\cdot\|\nabla_{\tau\theta}^2 U_{\tilde\theta}\|_{\sigma}\\
      &\leq \|\nabla_{\theta\tau}^2 U_{\tilde\theta}\|_{\sigma}\cdot\| (\nabla_{\tau\tau}^2 U_{\tilde\theta}^i)^{-1}\|_{\sigma}\cdot \|\nabla_{\tau\tau}^2 U_{\tilde\theta}-\nabla_{\tau\tau}^2 U_{\tilde\theta}^i\|_{\sigma}\cdot
      \|(\nabla_{\tau\tau}^2 U_{\tilde\theta})^{-1} \|_{\sigma}\cdot\|\nabla_{\tau\theta}^2 U_{\tilde\theta}\|_{\sigma}\\
     &\leq L^2\rho_{\tau}^i\frac{1}{\vert\lambda_{\max}(\nabla_{\tau\tau}^2 U_{\tilde\theta}^i)\cdot\lambda_{\max}(\nabla_{\tau\tau}^2 U_{\tilde\theta})\vert}
\end{align*}
Using bounds for $T_1,~T_2$ and $T_3$, we can obtain a bound on $\|S^i-S\|_{\sigma}\leq B^i$, where $B^i=\rho_{\theta}^i+ L\rho_{\theta\tau}^i\frac{1}{\vert\lambda_{\max}(\nabla_{\tau\tau}^2 U_{\tilde\theta}^i)\vert}+L\rho_{\tau\theta}^i\frac{1}{\vert\lambda_{\max}(\nabla_{\tau\tau}^2 U_{\tilde\theta}^i)\vert}+ L^2\rho_{\tau}^i\frac{1}{\vert\lambda_{\max}(\nabla_{\tau\tau}^2 U_{\tilde\theta}^i)\cdot\lambda_{\max}(\nabla_{\tau\tau}^2 U_{\tilde\theta})\vert}$.
Consider $\rho^i=\max\{\rho_{\theta}^i,~\rho_{\tau\theta}^i, ~\rho_{\theta\tau}^i,~\rho_{\tau}^i\}$. Hence, $B^i\le\rho^i\left(1+\frac{L}{\lambda_{\max}(\nabla_{\tau\tau}^2 U_{\tilde\theta}^i)}\left(2+\frac{1}{\lambda_{\max}(\nabla_{\tau\tau}^2 U_{\tilde\theta}}\right)\right)$.
Hence, we obtain
\begin{equation*}
  \lambda_{\min}(S^i)\geq -B^i+\beta,
\end{equation*}
where $\lambda_{\max}(S)=\beta$ such that $\beta\ge 0$.
Hence, we obtain $\left[\nabla_{\theta\theta}^2U_{\tilde\theta}^i-\nabla_{\theta\tau}^2U_{\tilde\theta}^i\left(\nabla_{\tau\tau}^2U_{\tilde\theta}^i\right)^{-1}\nabla_{\tau\theta}^2U_{\tilde\theta}^i\right](\hat\theta,\hat\tau)\succeq(\beta-B^i)I$.
When $\beta\geq B^i$, then  $S^i$ is positive semi-definite. 
When $B^i=0$, hence $\left[\nabla_{\theta\theta}^2U_{\tilde\theta}^i-\nabla_{\theta\tau}^2U_{\tilde\theta}^i\left(\nabla_{\tau\tau}^2U_{\tilde\theta}^i\right)^{-1}\nabla_{\tau\theta}^2U_{\tilde\theta}^i\right](\hat\theta,\hat\tau)\succeq \beta I$, thus it will be positive semidefinite. When $\rho_{\tau}^i<\alpha $ and $\beta>B^i$, then the suuficient condition for $\varepsilon^i$-approximate equilibrium is satisfied. And we obtain the result.

Thus, for each client $i$, any approximation error $\varepsilon^i$ that satisfies:
\begin{equation*}
 \max\{ \zeta_{\theta}^i, \zeta_{\tau}^i\}\le \varepsilon^i\le \min\{\alpha-\rho_{\tau}^i,\beta- B^i \}.  
\end{equation*}
for $\rho_{\tau}^i<\alpha$ and $B^i>\beta$, then $(\hat\theta,\hat\tau)$ is an $\varepsilon^i$-approximate local equilibrium point for client $i$.
\end{proof}

\subsection{Consistency}
\label{ap:consistency}
\subsubsection{Assumptions}
We first state the assumptions that are necessary to establish the consistency of the estimated parameter.

\begin{assumption}[Identification]
\label{asm:identification_consistency}
    $\theta_0$ is the unique $\theta\in\Theta$ such that $\psi(f^i;\theta)=0$ for all $f^i\in \gF$, where $i\in[n]$.
\end{assumption}

\begin{assumption}[Absolutely Star Shaped]
\label{asm:starshaped_consistency}
   For every $f^i\in\gF^i$ and $|c|\leq 1$, we have $cf^i\in \gF^i$.
\end{assumption}

\begin{assumption}[Continuity]
\label{asm:continuity_consistency}
    For any $x$, $g^i(x;\theta),~f^i(x;\tau)$ are continuous in $\theta$ and $\tau$, respectively for all $i\in[N]$.
\end{assumption}

\begin{assumption}[Boundedness]
\label{asm:boundedness_consistency}
    $Y^i$, $\sup_{\theta\in\Theta}|g^i(X;\theta)|,~\sup_{\tau\in\gT}|f^i(Z;\tau)|$ are bounded random variables for all $i\in[N]$.
\end{assumption}

\begin{assumption}[Bounded Complexity]
\label{asm:complexity_consistency}
    $\gF^i$ and $\gG^i$ have bounded Rademacher complexities:
    \begin{equation*}
        \frac{1}{2^{n_i}}
    \sum_{\xi_i\in\{-1,+1\}^{n_i}}\mathbb{E}\sup_{\tau\in\gT}\frac{1}{n_i}\sum_{k=1}^{n_i}\xi_if^i(Z_k;\tau)\rightarrow 0, \quad \frac{1}{2^{n_i}}\sum_{\xi_i\in\{-1,+1\}^{n_i}}\mathbb{E}\sup_{\theta\in\Theta}\frac{1}{n_i}\sum_{k=1}^{n_i}\xi_ig^i(X_k;\theta)\rightarrow 0.
    \end{equation*}
\end{assumption}
\subsubsection{Proof of Theorem~\ref{thm:consistency}}
\begin{theorem}[Restatement of of Theorem~\ref{thm:consistency}]
	Let $\tilde\theta_{n}$ be a data-dependent choice for the federated objective that has a limit in probability. For each client $i\in[N]$, define $        m^i(\theta,\tau,\tilde{\theta}):=f^i(Z^i;\tau)(Y^i-g(X^i;\theta))-\frac{1}{4}f^i(Z^i;\tau)^2(Y^i-g(X^i;\tilde{\theta}))^2$, $ M^i(\theta)=\sup_{\tau \in \gT}\mathbb{E}[m^i(\theta,\tau,\tilde{\theta})]$ and  $\eta^i(\epsilon):=inf_{d(\theta,\theta_0)\ge\epsilon}M^i(\theta)-M^i(\theta_0)$ for every $\epsilon>0$. Let $(\hat\theta_{n},\hat\tau_{n})$ be a solution that satisfies the approximate equilibrium for each of the client $i\in[N]$ as
	%  \begin{eqnarray*}
		\[
		\sup_{\tau\in\gT}  U_{\tilde\theta}^i(\hat\theta_{n},\tau)-\varepsilon^i - o_p(1)\leq  ~U^i_{\tilde\theta}(\hat\theta_{n},\hat\tau_{n})\leq ~\inf_{\theta\in\Theta} {\max_{{\tau\prime : \|\tau\prime-\hat\tau_{n}\|\leq h(\delta)}} U^i_{\tilde\theta}(\theta,\tau\prime)}+\varepsilon^i+o_p(1),
		\]
		% \end{eqnarray*}
	for some $\delta_0$, such that for any $\delta\in(0,\delta_0],$ and any $\theta,\tau$ such that $\|\theta-\hat\theta\|\leq \delta$ and $\|\tau-\hat\tau\|\leq\delta$  and a function $h(\delta)\rightarrow0$ as $\delta\rightarrow0$. Then, under similar assumptions as in Assumptions 1 to 5 of \cite{bennett2019deep}, the global solution $\hat\theta_{n}$ is a consistent estimator to the true parameter $\theta_0$, i.e. $\hat\theta_{n}\xrightarrow{p}\theta_0$ when the approximate error $\varepsilon^i<\frac{\eta^i(\epsilon)}{2}$ for every $\epsilon>0$ for each client $i\in[N]$.
\end{theorem} 

\begin{proof}
The proof follows from the result of \citet{bennett2019deep} that established the consistency of the \dgmm estimator. 

First, we define the following terms for the ease of analysis:
    \begin{align*}
        m^i(\theta,\tau,\tilde{\theta})&=f^i(Z^i;\tau)(Y^i-g(X^i;\theta))-\frac{1}{4}f^i(Z^i;\tau)^2(Y^i-g(X^i;\tilde{\theta}))^2\\
        M^i(\theta)&=\sup_{\tau \in \gT}\mathbb{E}[m^i(\theta,\tau,\tilde{\theta})]\\
        M_{n_i}(\theta)&=\sup_{\tau \in \gT}\mathbb{E}_{n_i}[m^i(\theta,\tau,\tilde{\theta}_n)]
    \end{align*}
    Note that $\tilde{\theta}_n$ is a data-dependent sequence for the global model. {Practically, the previous global iterate is used as $\tilde{\theta}$. Thus, we can define for the federated setting $\tilde{\theta}_n=\frac{1}{N}\sum_{i=1}^{N}\tilde{\theta}_{n_i}$.} Let's assume $\tilde{\theta}_n \xrightarrow{p} \tilde{\theta}$.
    
    \textbf{Claim 1:}\label{claim1} $\sup_{\theta}\vert M_{n_i}(\theta)-M^i(\theta)\vert\xrightarrow{p}0.$
    \begin{align*}
        \sup_{\theta}\vert M_{n_i}(\theta)-M^i(\theta)\vert&=  \sup_{\theta}\left\vert \sup_{\tau \in \gT}\mathbb{E}_{n_i}[m^i(\theta,\tau,\tilde{\theta}_{n})]-\sup_{\tau \in \gT}\mathbb{E}[m^i(\theta,\tau,\tilde{\theta})]\right\vert\\
        &\leq\sup_{\theta, \tau} \left\vert \mathbb{E}_{n_i}[m^i(\theta,\tau,\tilde{\theta}_{n})]-\mathbb{E}[m^i(\theta,\tau,\tilde{\theta})]\right\vert\\
        % &\leq\sup_{\theta, \tau} \frac{1}{N}\sum_{i=1}^{N}\left\vert \mathbb{E}_{n_i}[m^i(\theta,\tau,\tilde{\theta}_{n})]-\mathbb{E}[m^i(\theta,\tau,\tilde{\theta})]\right\vert\\
        % &\leq\sup_{\theta, \tau} \frac{1}{N}\sum_{i=1}^{N}\left\vert \mathbb{E}_{n_i}[m^i(\theta,\tau,\tilde{\theta}_{n})]-\mathbb{E}[m^i(\theta,\tau,\tilde{\theta})]\right\vert\\
        &\leq\sup_{\theta, \tau}\left\vert \mathbb{E}_{n_i}[m^i(\theta,\tau,\tilde{\theta}_{n})]-\mathbb{E}[m^i(\theta,\tau,\tilde{\theta}_n)]\right\vert
        + \sup_{\theta, \tau} \left\vert \mathbb{E}[m^i(\theta,\tau,\tilde{\theta}_{n})]-\mathbb{E}[m^i(\theta,\tau,\tilde{\theta})]\right\vert\\
        &\leq\sup_{\theta_1,\theta_2, \tau} \left\vert \mathbb{E}_{n_i}[m^i(\theta_1,\tau,\theta_{2})]-\mathbb{E}[m^i(\theta_1,\tau,\theta_2)]\right\vert
        + \sup_{\theta, \tau} \left\vert \mathbb{E}[m^i(\theta,\tau,\tilde{\theta}_{n})]-\mathbb{E}[m^i(\theta,\tau,\tilde{\theta})]\right\vert\\
    \end{align*}
    We will now handle the two terms in the above equation separately.

    We will take the first term and call it $B_1$. For $m^i(\theta,\tau,\tilde{\theta}_n)$, we constitute its empirical counterpart $m^i_k(\theta,\tau,\tilde{\theta}_n)=f^i(Z^i_k;\tau)(Y^i_k-g^i(X_k^i;\theta))-\frac{1}{4}f^i(Z^i_k;\tau)^2(Y^i_k-g^i(X_k^i;\tilde{\theta}))^2$ and using ${m^i_k}^{\prime}(\theta,\tau,{\tilde{\theta}}^{\prime}_n)$ with ghost variables $\tilde{\theta}^{\prime}_n$ for symmetrization and $\epsilon_k$ as $k$ i.i.d. Rademacher random variables , we obtain
    \begin{align*}
        \mathbb{E}[B_1]&=\mathbb{E}\left[\sup_{\theta_1,\theta_2, \tau} \left\vert\frac{1}{n_i}\sum_{k=1}^{n_i}m^i_k(\theta_1,\tau,\theta_{2})-\mathbb{E}\left[{m^i_k}^{\prime}(\theta_1,\tau,{{\theta}}^{\prime}_2)\right]\right\vert\right]\\
        &\le\mathbb{E}\left[\sup_{\theta_1,\theta_2, \tau} \left\vert\frac{1}{n_i}\sum_{k=1}^{n_i}\left(m^i_k(\theta_1,\tau,\theta_{2})-{m^i_k}^{\prime}(\theta_1,\tau,{{\theta}}^{\prime}_2)\right)\right\vert\right]\\
        &\le\mathbb{E}\left[\sup_{\theta_1,\theta_2, \tau} \left\vert\frac{1}{n_i}\sum_{k=1}^{n_i}\epsilon
        _k \left(m^i_k(\theta_1,\tau,\theta_{2})-{m^i_k}^{\prime}(\theta_1,\tau,{{\theta}}^{\prime}_2)\right)\right\vert\right]\\
        &\le2\mathbb{E}\left[\sup_{\theta_1,\theta_2, \tau} \left\vert\frac{1}{n_i}\sum_{k=1}^{n_i}\epsilon
        _k m^i_k(\theta_1,\tau,\theta_{2})\right\vert\right]\\
         &\le2\mathbb{E}\left[\sup_{\theta,\tau}\left\vert\frac{1}{n_i}\sum_{k=1}^{n_i}\epsilon
        _k f^i(Z^i_k;\tau)(Y^i_k-g^i(X_k^i;\theta))\right\vert\right]\\
        &\quad +\frac{1}{2}\mathbb{E}\left[\sup_{\theta,\tau} \left\vert \frac{1}{n_i}\sum_{k=1}^{n_i}\epsilon
        _k  f^i(Z^i_k;\tau)^2(Y^i_k-g^i(X_k^i;\tilde{\theta}))^2\right\vert\right]\\
        &\le2\mathbb{E}\left[\sup_{\theta,\tau}\left\vert\frac{1}{n_i}\sum_{k=1}^{n_i}\epsilon
        _k \left(\frac{1}{2}f^i(Z^i_k;\tau)^2+\frac{1}{2}(Y^i_k-g^i(X_k^i;\theta))^2\right)\right\vert\right]\\
        &\quad +\frac{1}{2}\mathbb{E}\left[\sup_{\theta,\tau} \left\vert \frac{1}{n_i}\sum_{k=1}^{n_i}\epsilon
        _k \left(\frac{1}{2} f^i(Z^i_k;\tau)^4+\frac{1}{2}(Y^i_k-g^i(X_k^i;\tilde{\theta}))^4\right)\right\vert\right]\\
        &\le\mathbb{E}\left[\sup_{\theta,\tau} \left\vert\frac{1}{n_i}\sum_{k=1}^{n_i}\epsilon
        _k f^i(Z^i_k;\tau)^2\right\vert\right]+\mathbb{E}\left[\sup_{\theta,\tau} \left\vert\frac{1}{n_i}\sum_{k=1}^{n_i}\epsilon_k(Y^i_k-g^i(X_k^i;\theta))^2\right\vert\right]\\
        &\quad +\frac{1}{4}\mathbb{E}\left[\sup_{\theta,\tau}\left\vert \frac{1}{n_i}\sum_{k=1}^{n_i}\epsilon
        _k  f^i(Z^i_k;\tau)^4\right\vert\right]+\frac{1}{4}\mathbb{E}\left[\sup_{\theta,\tau} \left\vert \frac{1}{n_i}\sum_{k=1}^{n_i}\epsilon
        _k (Y^i_k-g^i(X_k^i;\tilde{\theta}))^4\right\vert\right]
    \end{align*}
    Using boundedness assumption~\ref{asm:boundedness_consistency}, we consider the mapping from $f^i(Z^i_k;\tau)$ and $g^i(X_k^i;\tilde{\theta})$ to the summation terms in the last inequality as Lipschitz functions, hence for any functional class $\gF^i$ and $L$- Lipschitz function $\phi$, $\gR_{n_i}(\phi~\circ~f^i)\leq L\gR_{n_i}(\gF^i),$  where $\gR_{n_i}(\gF^i)$ is the Rademacher complexity of class $\gF^i$. Hence, $\mathbb{E}[B_1]\leq L(\gR_{n_i}(\gG^i)+\gR_{n_i}(\gF^i))$. Using assumption~\ref{asm:complexity_consistency}, $\mathbb{E}[B_1]\rightarrow 0.$ Let $B_1^\prime$ be a modified value of $B$, after changing the $j$-th value of $X^i,Z^i$ and $Y^i$ values, using assumption~\ref{asm:boundedness_consistency} on boundedness, we obtain the bounded difference inequality:
    \begin{align*}
        \sup_{X_{1:n_i}, Z_{1:n_i}, Y_{1:n_i}, X^\prime_j,Z^\prime_j,Y^\prime_j}|B_1-B_1^\prime|
        &\leq        \sup_{\theta_1,\theta_2,\tau, X_{1:n_i}, Z_{1:n_i}, Y_{1:n_i}, X^\prime_j,Z^\prime_j,Y^\prime_j}|\frac{1}{n_i}\left(m_j^i(\theta_1,\tau,\theta_2)-m_j^{i\prime}(\theta_1,\tau,\theta_2)\right)|\\
        \leq\frac{b}{n_i},
    \end{align*}
    where $b$ is some constant. Using McDiarmid's Inequality, we have $P(|B_1-\mathbb{E}[B_1]|\geq \epsilon_0)\leq 2\exp\left(\frac{-2n_i\epsilon_0^2}{c^2}\right)$. And $\mathbb{E}[B_1]\rightarrow 0$, we have $B_1\xrightarrow{p}0$.
    
    Now, we will handle $B_2$. For that
    \begin{align*}
B_2&=\sup_{\theta,\tau}\left\vert\mathbb{E}\left[m^i(\theta,\tau,\tilde{\theta}_n)\right]-\mathbb{E}\left[m^i(\theta,\tau,\tilde{\theta})\right]\right\vert\\
&=\sup_{\theta,\tau}\left\vert\mathbb{E}\left[f^i(Z^i;\tau)(Y^i-g(X^i;\theta))-\frac{1}{4}f^i(Z^i;\tau)^2(Y^i-g(X^i;\tilde{\theta}_n))^2\right]\right.\\
&\qquad-\left.\mathbb{E}\left[f^i(Z^i;\tau)(Y^i-g(X^i;\theta))-\frac{1}{4}f^i(Z^i;\tau)^2(Y^i-g(X^i;\tilde{\theta}))^2\right]\right\vert \\
&=\sup_{\theta,\tau}\frac{1}{4}\left\vert\mathbb{E}\left[f^i(Z^i;\tau)^2(Y^i-g(X^i;\tilde{\theta}_n))^2\right]-\mathbb{E}\left[f^i(Z^i;\tau)^2(Y^i-g(X^i;\tilde{\theta}))^2\right]\right\vert \\
&=\sup_{\theta,\tau}\frac{1}{4}\left\vert\mathbb{E}\left[f^i(Z^i;\tau)^2(Y^i-g(X^i;\tilde{\theta}_n))^2\right]+\mathbb{E}\left[f^i(Z^i;\tau)^2(Y^i-g(X^i;\tilde{\theta}))^2\right]\right.\\
&\qquad-\left.\mathbb{E}\left[f^i(Z^i;\tau)^2(Y^i-g(X^i;\tilde{\theta}))^2\right]-\mathbb{E}\left[f^i(Z^i;\tau)^2(Y^i-g(X^i;\tilde{\theta}))^2\right]\right\vert \\
&\leq\frac{1}{4}\sup_{\tau}\left\vert\mathbb{E}\left[f^i(Z^i;\tau)^2\omega_n\right]\right\vert \\
    \end{align*}
    Here, $\omega_n=\left|(Y^i-g(X^i;\tilde{\theta}_n))^2-(Y^i-g(X^i;\tilde{\theta}))^2\right|$. 
    Due to our assumption, $\tilde{\theta}_n\xrightarrow{p}\tilde{\theta}$, thus $\omega_n\xrightarrow{p}0$ due to Slutsky's and continuous mapping theorem.
    Since, $f^i(Z;\tau)$ is uniformly bounded, thus for some constant $b^\prime>0$, we have
    \begin{align*}
B_2 &\leq\frac{b^\prime}{4}\sup_{\tau}\frac{1}{N}\sum_{i=1}^{N}\left\vert\mathbb{E}\left[\omega_n\right]\right\vert \\
&\leq\frac{b^\prime}{4}\sup_{\tau}\frac{1}{N}\sum_{i=1}^{N}\mathbb{E}\left[\left\vert\omega_n\right\vert\right]
 \end{align*}
 Based on the boundedness assumption, we can verify that $\omega_n$ is bounded, hence using Lebesgue Dominated Convergence Theorem, we can conclude that $\mathbb{E}\left[\left\vert\omega_n\right\vert\right]\rightarrow 0.$

 Thus, using the convergence of $B_1$ and $B_2$, we have $\sup_\theta\vert M_{n_i}(\theta)-M^i(\theta)\vert\xrightarrow{p}0$ for each $i\in [N]$.\\

\noindent\textbf{Claim 2:} for every $\epsilon>0$, we have $\inf_{d(\theta,\theta_0)\geq \epsilon}M^i(\theta)>M^i(\theta_0)$.

$M^i(\theta_0)$ is the unique minimizer of $M^i(\theta)$. By assumption~(\ref{asm:identification_consistency}) and~(\ref{asm:starshaped_consistency}), $\theta_0$ is the unique minimizer of $\sup_{\tau}\mathbb{E}[f^i(Z^i;\tau)(Y^i-g^i(X;\theta))]$ such that $\sup_{\tau}\mathbb{E}[f^i(Z^i;\tau)(Y^i-g^i(X;\theta))]=0$. Thus, any other value of $\theta$ will have at least one $\tau$ such that this expectation is strictly positive. $M(\theta_0)=0$ and $M(\theta_0)=\sup_{\tau}-\frac{1}{4}f^i(Z^i;\tau)^2(Y^i-g^i(X;\theta))^2$, the function whose supremum is being evaluated is non-positive but can be set to zero by assumption~(\ref{asm:starshaped_consistency}) by taking the zero function of $f^i$. Let for any other $\theta^{\prime}\neq\theta_0$, let ${f^i}^{\prime}$ be a function in $\gF^i$ such that $\mathbb{E}[f^i(Z)(Y^i-g^i(X;\theta^{\prime}))]>0$. If we have $\mathbb{E}[{f^i}^{\prime}(Z)^2(Y^i-g^i(X;\tilde\theta))^2]=0$, then $M^i(\theta^\prime)>0$. Else, consider $c{f^i}^{\prime}$ for any $c\in(0,1)$. Using assumption~(\ref{asm:starshaped_consistency}), $c{f^i}^{\prime}\in\gF^i$, thus
\begin{eqnarray*}
    M^i(\theta^\prime)&=\sup_{f^i\in\gF^i}\mathbb{E}\left[f^i(Z^i)(Y^i-g(X^i;\theta^{\prime}))-\frac{1}{4}f^i(Z^i)^2(Y^i-g(X^i;\tilde{\theta}))^2\right]\\
    &\leq c\mathbb{E}\left[{f^i}^{\prime}(Z^i)(Y^i-g(X^i;\theta^{\prime}))\right]-\frac{c^2}{4}\mathbb{E}\left[{f^i}^{\prime}(Z^i)^2(Y^i-g(X^i;\tilde{\theta}))^2\right]
\end{eqnarray*}
This is quadratic in $c$ and is positive when $c$ is sufficiently small, thus $M^i(\theta^\prime)>0$.

We now prove claim 2 using contradiction. Let us assume claim 2 is false, i.e. for some $\epsilon>0$, we have $\inf_{\theta\in \textit{B}(\theta_0,\epsilon)}M^i(\theta)=M^i(\theta_0),$ where $\textit{B}(\theta_0,\epsilon)^c=\{\theta~|\quad d(\theta,\theta_0)\geq \epsilon\}$., since $\theta_0$ is the unique minimizer of $M^i(\theta)$ by assumption~(\ref{asm:identification_consistency}). Thus, there must exist some sequence $(\theta_1,\theta_2,\dots)$ in $\textit{B}(\theta_0,\epsilon)^c$ such that $M^i(\theta_n)\rightarrow M^i(\theta_0)$. By construction, $\textit{B}(\theta_0,\epsilon)^c$ is closed and the corresponding limit parameters $\theta^*=\lim_{n\rightarrow\infty}\theta_n\in\textit{B}(\theta_0,\epsilon)^c$ must satisfy $M^i(\theta^*)=M^i(\theta_0)$ using assumption~(\ref{asm:continuity_consistency}). But $d(\theta^*,\theta_0)\geq \epsilon>0,$ thus $\theta^*\neq\theta_0$. This contradicts that $\theta_0$ is the unique minimizer of $M^i(\theta)$; hence, claim 2 is true.

\noindent\textbf{Claim 3:} For the third part, we know that $\hat\theta_{n}$ satisfies the $\varepsilon^i$- approximate equilibrium condition, given as:

         \begin{equation*}
        \mathbb{E}_{n_i}[m^i(\hat\theta_{n},\tau,\tilde\theta_{n})]-\varepsilon^i\leq  \mathbb{E}_{n_i}[m^i(\hat\theta_{n},\hat\tau_{n},\tilde\theta_{n})]\leq {\max_{{\tau\prime : \|\tau\prime-\hat\tau_{n}\|\leq h(\delta)}} \mathbb{E}_{n_i}[m^i(\theta,\tau\prime,\tilde\theta_{n})]}+\varepsilon^i,
     \end{equation*}
     for a function $h(\delta)\rightarrow0$ as $\delta\rightarrow0$  and some $\delta_0$, such that for any $\delta\in(0,\delta_0],$ and any $\theta,\tau$ such that $\|\theta-\hat\theta\|\leq \delta$ and $\|\tau-\hat\tau\|\leq\delta$.
Assume that this is true with $o_p(1)$, hence
     \begin{equation*}
      \sup_{\tau}  \mathbb{E}_{n_i}[m^i(\hat\theta_{n},\tau,\tilde\theta_{n})]-\varepsilon^i - o_p(1)\leq  \mathbb{E}_{n_i}[m^i(\hat\theta_{n},\hat\tau_{n},\tilde\theta_{n})]\leq \inf_{\theta} {\max_{{\tau\prime : \|\tau\prime-\hat\tau_{n}\|\leq h(\delta)}} \mathbb{E}_{n_i}[m^i(\theta,\tau\prime,\tilde\theta_{n})]}+\varepsilon^i+o_p(1),
     \end{equation*}.

Now, since $M_{n_i}(\hat\theta_{n})= sup_{\tau}  \mathbb{E}_{n_i}[m^i(\hat\theta_{n},\tau,\tilde\theta_{n})]$. Hence,
\begin{equation*}
inf_{\theta} {\max_{{\tau\prime : \|\tau\prime-\hat\tau_{n}\|\leq h(\delta)}} \mathbb{E}_{n_i}[m^i(\theta,\tau\prime,\tilde\theta_{n})}\leq inf_{\theta} sup_{\tau} \mathbb{E}_{n_i}[m^i(\theta,\tau\prime,\tilde\theta_{n})]=inf_{\theta} M_{n_i}(\theta)\leq M_{n_i}(\theta_0)
\end{equation*}
Thus, we have
  \begin{equation*}
    M_{n_i}(\hat\theta_{n})-\varepsilon^i - o_p(1)\leq  \mathbb{E}_{n_i}[m^i(\hat\theta_n,\hat\tau_n,\tilde\theta_{n})]\leq M_{n_i}(\theta_0)+\varepsilon^i +o_p(1).
     \end{equation*}
We have proven all three conditions until now.
From the first and second condition, since $|M_{n_i}(\theta_0)-M^i(\theta_0)|\xrightarrow{p} 0$, hence $ M_{n_i}(\hat\theta_{n})\leq M^i(\theta_0)+2\varepsilon^i + o_p(1)$.
Hence, we obtain
\begin{align*}
M^i(\hat\theta_{n})-M^i(\theta_0)&\leq M^i(\hat\theta_{n})-M_{n_i}(\hat\theta_{n})+2\varepsilon^i + o_p(1)\\
&\leq \sup_{\theta}|M^i(\hat\theta)-M_{n_i}(\hat\theta)|+2\varepsilon^i + o_p(1)\\
& \leq 2\varepsilon^i+ o_p(1)
\end{align*}
% Since, let $\eta(\epsilon):=inf_{d(\theta,\theta_0)\geq \epsilon} M^i(\theta)-M^i(\theta_0)$. Hence, whenever $d(\hat\theta_{n},\theta_0)\geq \epsilon$, we have $M^i(\hat\theta_{n})-M^i(\theta_0)\geq \eta(\epsilon)$. Thus, $\mathbb{P}[d(\hat\theta_{n},\theta_0)\geq \epsilon]\leq \mathbb{P}[M^i(\hat\theta_{n})-M^i(\theta_0)\geq \eta(\epsilon)]$. For every $\epsilon>0$, we have $\eta(\epsilon)>0$ from claim 2, and $M^i(\hat\theta_{n})-M^i(\theta_0)=2\varepsilon^i + o_p(1)$. We have that for every $\epsilon>0$ and sufficiently small $\varepsilon^i$, the RHS probability converges to $0$, thus $d(\hat\theta_{n},\theta_0)=o_p(1)$, hence $\hat\theta_{n}$ converges to $\theta_0$ for each client $i\in[N]$, when $\varepsilon^i$ is sufficiently small. When heterogeneity is zero, then $\hat\theta_{n}$ converges to $\theta_0$ in probability.\\

Hence, we obtain
\begin{align*}
M^i(\hat\theta_{n})-M^i(\theta_0)-2\varepsilon^i&\leq M^i(\hat\theta_{n})-M_{n_i}(\hat\theta_{n}) + o_p(1)\\
&\leq \sup_{\theta}|M^i(\hat\theta)-M_{n_i}(\hat\theta)| + o_p(1)\\
& \leq o_p(1)
\end{align*}
Since, let $\eta^i(\epsilon):=inf_{d(\theta,\theta_0)\geq \epsilon} M^i(\theta)-M^i(\theta_0)$. Hence, whenever $d(\hat\theta_{n},\theta_0)\geq \epsilon$, we have $M^i(\hat\theta_{n})-M^i(\theta_0)\geq \eta^i(\epsilon)$. Thus, $\mathbb{P}[d(\hat\theta_{n},\theta_0)\geq \epsilon]\leq \mathbb{P}[M^i(\hat\theta_{n})-M^i(\theta_0)\geq \eta^i(\epsilon)]=\mathbb{P}[M^i(\hat\theta_{n})-M^i(\theta_0)-2\varepsilon^i\geq \eta^i(\epsilon)-2\varepsilon^i]$. For every $\epsilon>0$, we have $\eta^i(\epsilon)>0$ from claim 2, and $M^i(\hat\theta_{n})-M^i(\theta_0)-2\varepsilon^i= o_p(1)$. Thus, $\eta^i(\epsilon)-2\varepsilon^i>0$ when $\varepsilon^i<\frac{\eta^i(\epsilon)}{2}$. We have that for every $\epsilon>0$ and $\varepsilon^i<\frac{\eta^i(\epsilon)}{2}$, the RHS probability converges to $0$, thus $d(\hat\theta_{n},\theta_0)=o_p(1)$, hence $\hat\theta_{n}$ converges in probability to $\theta_0$ for each client $i\in[N].$

\end{proof} 

\section{Limit Points of \fgda }

We first discuss the $\gamma$- \fgda flow.
\subsection{\fgda Flow}
\label{ap:fgdaflow}
The \fgda updates can be written as
\begin{align*}
\theta_{t+1}=\theta_{t}-\eta\frac{1}{\gamma}\frac{1}{N}\sum_{i\in[N]}\sum_{r=1}^{R}\left(\nabla_{\theta}U_{\tilde\theta}(\theta_{t}, \tau_{t})+(\nabla_{\theta}U_{\tilde\theta}^i(\theta_{t,r}^i, \tau_{t,r}^i)-\nabla_{\theta}U_{\tilde\theta}^i(\theta_{t}, \tau_{t}))\right.\\
\left.+(\nabla_{\theta}U_{\tilde\theta}^i(\theta_{t}, \tau_{t})-\nabla_{\theta}U_{\tilde\theta}(\theta_{t}, \tau_{t}))\right)\\
\tau_{t+1}=\tau_{t}+\eta\frac{1}{N}\sum_{i\in[N]}\sum_{r=1}^{R}\left(\nabla_{\tau}U_{\tilde\theta}(\theta_{t}, \tau_{t})+(\nabla_{\tau}U_{\tilde\theta}^i(\theta_{t,r}^i, \tau_{t,r}^i)-\nabla_{\tau}U_{\tilde\theta}^i(\theta_{t}, \tau_{t}))\right.\\\left.+(\nabla_{\tau}U_{\tilde\theta}^i(\theta_{t}, \tau_{t})-\nabla_{\tau}U_{\tilde\theta}(\theta_{t}, \tau_{t}))\right)
\end{align*}
Rearranging the terms and taking the continuous-time limit as $\eta \to 0$
\begin{align*}
   \lim_{\eta\rightarrow0}\frac{\theta_{t+1}-\theta_{t}}{\eta}=& \lim_{\eta\rightarrow0}-\frac{1}{\gamma}\frac{1}{N}\sum_{i\in[N]}\sum_{r=1}^{R} \left(\nabla_{\theta}U_{\tilde\theta}(\theta_{t}, \tau_{t})+(\nabla_{\theta}U_{\tilde\theta}^i(\theta_{t,r}^i, \tau_{t,r}^i)-\nabla_{\theta}U_{\tilde\theta}^i(\theta_{t}, \tau_{t}))\right.\\&\left.+(\nabla_{\theta}U_{\tilde\theta}^i(\theta_{t}, \tau_{t})-\nabla_{\theta}U_{\tilde\theta}(\theta_{t}, \tau_{t}))\right)\\
   \lim_{\eta\rightarrow0}\frac{\tau_{t+1}-\tau_{t}}{\eta}=& \lim_{\eta\rightarrow0}\frac{1}{N}\sum_{i\in[N]}\sum_{r=1}^{R}\left(\nabla_{\tau}U_{\tilde\theta}(\theta_{t}, \tau_{t})+(\nabla_{\tau}U_{\tilde\theta}^i(\theta_{t,r}^i, \tau_{t,r}^i)-\nabla_{\tau}U_{\tilde\theta}^i(\theta_{t}, \tau_{t}))\right.\\&\left.+(\nabla_{\tau}U_{\tilde\theta}^i(\theta_{t}, \tau_{t})-\nabla_{\tau}U_{\tilde\theta}(\theta_{t}, \tau_{t}))\right)
\end{align*}
% Under the smoothness of $U_{\tilde\theta}^i$ by Assumption~\ref{asm:twicediff} and for small $R$, we can take $\nabla_{\tau}U_{\tilde\theta}^i(\theta_{t,r}^i, \tau_{t,r}^i)\approx\nabla_{\tau}U_{\tilde\theta}^i(\theta_{t}, \tau_{t})$, thus \fgda-flow becomes:
We obtain the gradient flow equations as
\begin{align}
    \frac{d\theta}{dt} &= - \frac{R}{\gamma} \frac{1}{N} \sum_{i \in [N]} \left( \nabla_{\theta} U_{\tilde\theta}(\theta(t), \tau(t)) \right) - \frac{R}{\gamma} \frac{1}{N} \sum_{i \in [N]} \left( \nabla_{\theta} U_{\tilde\theta}^i(\theta^i(t), \tau^i(t)) - \nabla_{\theta} U_{\tilde\theta}^i(\theta(t), \tau(t))\right)\nonumber\\
    & \quad-\frac{R}{\gamma} \frac{1}{N} \sum_{i \in [N]} \left(\nabla_{\theta}U_{\tilde\theta}^i(\theta(t), \tau(t))-\nabla_{\theta}U_{\tilde\theta}(\theta(t), \tau(t)))\right), \label{eq:grad_flow_theta}\\
    \frac{d\tau}{dt} &= R\frac{1}{N} \sum_{i \in [N]} \left( \nabla_{\tau} U_{\tilde\theta}(\theta(t), \tau(t)) \right) + R\frac{1}{N} \sum_{i \in [N]} \left( \nabla_{\tau} U_{\tilde\theta}^i(\theta^i(t), \tau^i(t)) - \nabla_{\tau} U_{\tilde\theta}^i(\theta(t), \tau(t)) \right)\nonumber\\
    &\quad+R\frac{1}{N} \sum_{i \in [N]}\left(\nabla_{\tau}U_{\tilde\theta}^i(\theta(t), \tau(t))-\nabla_{\tau}U_{\tilde\theta}(\theta(t), \tau(t))\right). \label{eq:grad_flow_tau}
\end{align}

Using Assumption~\ref{asm:gradientheterogeneity}
\begin{align*}
   \left \|\frac{R}{\gamma} \frac{1}{N} \sum_{i \in [N]}(\nabla_{\theta}U_{\tilde\theta}^i(\theta(t), \tau(t))-\nabla_{\theta}U_{\tilde\theta}(\theta(t), \tau(t)))\right\|\leq \frac{R}{\gamma}  \zeta_{\theta}\\
     \left\|R\frac{1}{N} \sum_{i \in [N]}(\nabla_{\tau}U_{\tilde\theta}^i(\theta(t), \tau(t))-\nabla_{\tau}U_{\tilde\theta}(\theta(t), \tau(t)))\right\|\leq R\zeta_{\tau}
\end{align*}

Thus,
\begin{align*}
\frac{R}{\gamma} \frac{1}{N} \sum_{i \in [N]} (\nabla_{\theta}U_{\tilde\theta}^i(\theta(t), \tau(t))-\nabla_{\theta}U_{\tilde\theta}(\theta(t), \tau(t)))&=\gO\left(\frac{R}{\gamma}  \zeta_{\theta}\right)\\
   R \frac{1}{N} \sum_{i \in [N]} (\nabla_{\tau}U_{\tilde\theta}^i(\theta(t), \tau(t))-\nabla_{\tau}U_{\tilde\theta}(\theta(t), \tau(t)))&=\gO(R\zeta_{\tau})
\end{align*}
Since $U_{\tilde\theta}^i$ is Lipschitz smooth by assumption~\ref{asm:lipschitz}, we have
\begin{align*}
    \left\| \frac{R}{\gamma} \frac{1}{N} \sum_{i \in [N]}( \nabla_{\theta} U_{\tilde\theta}^i(\theta^i(t), \tau^i(t)) - \nabla_{\theta} U_{\tilde\theta}^i(\theta(t), \tau(t)))\right\|
    &\leq L\frac{R}{\gamma} \frac{1}{N} \sum_{i \in [N]}\|(\theta^i(t),\tau^i(t))-(\theta(t),\tau(t)\|,\\
    \left\|R\frac{1}{N} \sum_{i \in [N]} (\nabla_{\tau} U_{\tilde\theta}^i(\theta^i(t), \tau^i(t)) - \nabla_{\tau} U_{\tilde\theta}^i(\theta(t), \tau(t))) \right\| &\leq L R\frac{1}{N} \sum_{i \in [N]} \|(\theta^i(t),\tau^i(t))-(\theta(t),\tau(t))\|. 
\end{align*}

Substituting these bounds into Equations (\ref{eq:grad_flow_theta}) and (\ref{eq:grad_flow_tau}), we obtain

\begin{align*}
   \frac{R}{\gamma} \frac{1}{N} \sum_{i \in [N]}( \nabla_{\theta} U_{\tilde\theta}^i(\theta^i(t), \tau^i(t)) - \nabla_{\theta} U_{\tilde\theta}^i(\theta, \tau)) &= \mathcal{O}\left( L\frac{R}{\gamma} \frac{1}{N} \sum_{i \in [N]}\|(\theta^i(t),\tau^i(t))-(\theta(t),\tau(t)\| \right),\\
    R\frac{1}{N} \sum_{i \in [N]} (\nabla_{\tau} U_{\tilde\theta}^i(\theta^i(t), \tau^i(t)) - \nabla_{\tau} U_{\tilde\theta}^i(\theta, \tau))  &= \mathcal{O}\left(L R\frac{1}{N} \sum_{i \in [N]} \|(\theta^i(t),\tau^i(t))-(\theta(t),\tau(t))\| \right).
\end{align*}

Since the local update follows

\begin{align*}
    \theta^i(t) &= \theta(t)- \frac{\eta}{\gamma} \sum_{j=1}^{R} \nabla_{\theta} U_{\tilde\theta}^i(\theta_{j}^i(t), \tau_{j}^i(t)), \\
    \tau^i(t) &= \tau(t) +  \eta \sum_{j=1}^{R} \nabla_{\tau} U_{\tilde\theta}^i(\theta_{j}^i(t), \tau_{j}^i(t)),
\end{align*}

Using bounded gradient assumption, i.e. $\|\nabla_{\theta}U^i_{\tilde\theta}(\theta,\tau))\|^2\leq G_{\theta}$ and $\|\nabla_{\tau}U^i_{\tilde\theta}(\theta,\tau))\|^2\leq G_{\tau}$ for all $i$, as $\eta \to 0$ and $R$ is fixed and finite, the deviation $\|(\theta^i(t),\tau^i(t))-(\theta(t),\tau(t))\|$ vanish, leading to

\begin{align*}
    \frac{d\theta}{dt} &= - \frac{1}{\gamma} {R}\nabla_{\theta} U_{\tilde\theta}(\theta(t), \tau(t))+ \gO\left(\frac{R}{\gamma}  \zeta_{\theta}\right), \\
    \frac{d\tau}{dt} &=  R \nabla_{\tau} U_{\tilde\theta}(\theta(t), \tau(t))+ \gO(R\zeta_{\tau}).
\end{align*}

\subsection{Proof of Theorem~\ref{thm:inftyfgda}}
\label{ap:limitpoint}
\begin{proof}
Let $\mA=\nabla_{\theta\theta}^2 U_{\tilde\theta}(\theta, \tau) , \mB= \nabla_{\tau\tau}^2 U_{\tilde\theta}(\theta, \tau)$ and $\mC=\nabla_{\theta\tau}^2 U_{\tilde\theta}(\theta, \tau).$ Consider $\epsilon=\frac{1}{\gamma}$, thus for sufficiently small $\epsilon$ (hence a large $\gamma$), the Jacobian $\mJ$ of \fgda for a point $(\theta,\tau)$ is given as:
    \begin{equation*}
        \mJ_{\epsilon}=R\begin{pmatrix}
            -{\epsilon}\mA &-{\epsilon}\mC\\
            \mC^\top & \mB   
        \end{pmatrix}.
    \end{equation*}

Using Lemma~\ref{lem:rootsorder}, $ \mJ_{\epsilon}$ has $d_1 + d_2$ complex eigenvalues $\{\emLambda_j\}_{j=1}^{d_1+d_2}$ such that
\begin{equation}
  \begin{aligned}
        |\emLambda_j+\epsilon\mu_j|=o(\epsilon)&\qquad 1\leq j\leq d_1\\
        |\emLambda_{j+d_1}-\nu_j|=o(1),&\qquad 1\leq j\leq d_2,
    \end{aligned}
    \label{eqn:eigenvalue_block}
    \end{equation}
    where $\{\mu_j\}_{j=1}^{d_1}$ and $\{\nu_j\}_{j=1}^{d_2}$ are the eigenvalues of matrices $R(\mA-\mC\mB^{-1}\mC^\top)$ and $R\mB$ respectively.
    
    We now prove the theorem statement:
    
    $\mathcal{L}\text{oc}\gM\text{inimax}\subset\underline{\infty-\mathcal{FGDA}}\subset\overline{\infty-\mathcal{FGDA}}\subset\mathcal{L}\text{oc}\gM\text{inimax}\cup\{(\theta,\tau)|(\theta,\tau)\text{ is stationary and }\nabla_{\tau\tau}^2U_{\tilde\theta}(\theta, \tau)\text{ is degenerate}\}.$

    By definition of $\limsup$ and $\liminf$, we know that $\underline{\infty-\mathcal{FGDA}}\subset\overline{\infty-\mathcal{FGDA}}$.

    Now we show $\mathcal{L}\text{oc}\gM\text{inimax}\subset\underline{\infty-\mathcal{FGDA}}$. Consider a strict local minimax point $(\theta,\tau)$, then by sufficient condition it follows that:
    \begin{equation*}
        \mB\prec 0,\quad\text{and }\quad \mA-\mC\mB^{-1}\mC^\top\succ 0. 
    \end{equation*}
    Thus, $R\mB\prec 0,\text{ and }R(\mA-\mC\mB^{-1}\mC^\top)\succ 0$, where $R$ is always positive. Hence, $\{\nu_j\}_{j=1}^{d_1}<0$ and $\{\mu_j\}_{j=1}^{d_2}<0$. Using equations~\ref{eqn:eigenvalue_block}, for some small $\epsilon_0<\epsilon$,  $\operatorname{Re}(\emLambda_j) < 0 \text{ for all } j.$ Thus, $(\theta,\tau)$ is a strict linearly stable point of $\frac{1}{\epsilon}$-\fgda.

    Now, we show $\overline{\infty-\mathcal{FGDA}}\subset\mathcal{L}\text{oc}\gM\text{inimax}\cup\{(\theta,\tau)|(\theta,\tau)\text{ is stationary and }\nabla_{\tau\tau}^2U_{\tilde\theta}(\theta, \tau)\text{ is degenerate}\}.$
    Consider  $(\theta,\tau)$ a strict linearly stable point of $\frac{1}{\epsilon}$-\fgda, such that for some small $\epsilon$,  $\operatorname{Re}(\emLambda_j) < 0 \text{ for all } j.$ By equation~\ref{eqn:eigenvalue_block}, assuming $B^{-1}$ exists
    \begin{equation*}
        R\mB\prec 0,\quad\text{and }\quad R(\mA-\mC\mB^{-1}\mC^\top)\succeq 0. 
    \end{equation*}
    Since, $R$ is positive, thus $\mB\prec 0,\text{ and }\mA-\mC\mB^{-1}\mC^\top\succeq 0.$ Let's assume $\mA-\mC\mB^{-1}\mC^\top$ has $0$ as an eigenvalue. Thus, there exists a unit eigenvector $\vw$ such that $\mA-\mC\mB^{-1}\mC^\top\vw=0$. Then,
    \begin{align*}
            \mJ_{\epsilon}\cdot(\vw,-B^{-1}C^\top\vw)^\top=
       R \begin{pmatrix}
       -\epsilon\mA & -\epsilon \mC\\
       \mC^\top & \mB
    \end{pmatrix}
    \cdot
    \begin{pmatrix}
        \vw \\-B^{-1}C^\top\vw
    \end{pmatrix}=\mathbf{0}.
    \end{align*}
    Thus, $\mJ_{\epsilon}$ has $0$ as its eigenvalue, which is a contradiction because for strict linearly stable point $\operatorname{Re}(\emLambda_j) < 0 \text{ for all } j$. Thus, $\mA-\mC\mB^{-1}\mC^\top\succ0$. Hence, $(\theta,\tau)$ is a strict local minimax point.

    Let $G: \mathbb{R}^d \times \mathbb{R}^k \to \mathbb{R}$ be the function defined as:
$G(\theta, \tau) = \det(\nabla_{\tau\tau}^2 U_{\tilde\theta}(\theta, \tau)).$
Let's assume that $\nabla_{\tau\tau}^2 U_{\tilde\theta}(\theta, \tau)$ is smooth, thus the determinant function is a polynomial in the entries of the Hessian, which implies that $G$ is a smooth function. Since $\nabla_{\tau\tau}^2 U_{\tilde\theta}(\theta, \tau)=0$ implies at least one eigenvalue of $\nabla_{\tau\tau}^2 U_{\tilde\theta}(\theta, \tau)$ is zero, thus $\det(\nabla_{\tau\tau}^2 U_{\tilde\theta}(\theta, \tau))=0$.

We aim to show that the set 
$$\gA = \{ (\theta, \tau) \mid (\theta, \tau) \text{ is stationary and } \det(\nabla_{\tau\tau}^2 U_{\tilde\theta}(\theta, \tau)) = 0 \}$$
has measure zero in $\mathbb{R}^{d} \times \mathbb{R}^k$.

A point $q \in \mathbb{R}^{d} \times \mathbb{R}^k$ is a \emph{regular value} of $G$ if for every $(\theta, \tau) \in G^{-1}(q)$, the differential $dG(\theta, \tau)$ is surjective. Otherwise, $q$ is a \emph{critical value}.

The differential of $G$ is given by: $\nabla G(\theta, \tau) = \text{Tr} \left( \text{Adj}(\nabla_{\tau\tau}^2 U_{\tilde\theta}) \cdot \nabla (\nabla_{\tau\tau}^2 U_{\tilde\theta}) \right).$
If $\det(\nabla_{\tau\tau}^2 U_{\tilde\theta}(\theta, \tau)) = 0$, then the Hessian $\nabla_{\tau\tau}^2 U_{\tilde\theta}$ is singular. This causes its adjugate matrix to lose rank, leading to a degeneracy in $\nabla G(\theta, \tau)$, making $dG(\theta, \tau)$ \emph{not surjective}.

Thus, every $(\theta, \tau)$ satisfying $G(\theta, \tau) = 0$ is a critical point of $G$, meaning that $0$ is a \emph{critical value} of $G$.

By Sard’s theorem, the set of critical values of a smooth function has measure zero in the codomain. Since $G$ is smooth, the set of critical values of $G$ in $\mathbb{R}$ has measure zero. In particular, since $0$ is a critical value of $G$, the set:
$G^{-1}(0) = \{ (\theta, \tau) \mid \det(\nabla_{\tau\tau}^2 U_{\tilde\theta}(\theta, \tau)) = 0 \}$
has measure zero in $\mathbb{R}^{d+k}$.

Since the set of degenerate $\nabla_{\tau\tau}^2 U_{\tilde\theta}(\theta, \tau)$ is precisely $G^{-1}(0)$, we conclude that
$\text{Lebesgue measure}(\gA) = 0.$ Thus, the set of stationary points where the Hessian $\nabla_{\tau\tau}^2 U_{\tilde\theta}(\theta, \tau)$ is singular has measure zero in $\mathbb{R}^{d}\times \mathbb{R}^k$.
    \end{proof}
\begin{lemma}\cite{mishael_lemma}
\label{rootsmishael}
\textit{Given a polynomial} 
$
p_n(z) := \sum_{k=0}^{n} a_k z^k,
$
\textit{where} $ a_n \neq 0 $, \textit{an integer} $ m \geq n $ \textit{and a number} $ \epsilon > 0 $, 
\textit{there exists a number} $ \delta > 0 $ \textit{such that whenever the} $ m+1 $ \textit{complex numbers} $ b_k $, $ 0 \leq k \leq m $, 
\textit{satisfy the inequalities}
\begin{equation*}
|b_k - a_k| < \delta \quad \text{for } 0 \leq k \leq n, \quad \text{and} \quad |b_k| < \delta \quad \text{for } n+1 \leq k \leq m,
\end{equation*}
then the roots $ \beta_k $, $ 1 \leq k \leq m $, of the polynomial
$
q_m(z) := \sum_{k=0}^{m} b_k z^k
$
can be labeled in such a way as to satisfy, with respect to the zeros $ \alpha_k $, $ 1 \leq k \leq n $, \textit{of} $ p_n(z) $, the inequalities
\begin{equation*}
|\beta_k - \alpha_k| < \epsilon \quad \text{for } 1 \leq k \leq n, \quad \text{and} \quad |\beta_k| > 1/\epsilon \quad \text{for } n+1 \leq k \leq m.
\end{equation*}
\end{lemma}

\begin{lemma}
\label{lem:rootsorder}
    For any symmetric matrix $\mA\in\mathbb{R}^{d_1 \times d_1},~\mB\in\mathbb{R}^{d_2 \times d_2}$, any rectangular matrix $\mC\in\mathbb{R}^{d_1 \times d_2}$ and a scalar $R$, assume that $\mB$ is non-degenerate. Then, matrix
    \begin{equation*}
   R \begin{pmatrix}
       -\epsilon\mA & -\epsilon \mC\\
       \mC^\top & \mB
    \end{pmatrix}
    \end{equation*}
    has $d_1 + d_2$ complex eigenvalues $\{\emLambda_j\}_{j=1}^{d_1 + d_2}$ with following form for sufficiently small $\epsilon$:
    \begin{align*}
        |\emLambda_j+\epsilon\mu_j|=o(\epsilon)&\qquad 1\leq j\leq d_1\\
        |\emLambda_{j+d_1}-\nu_j|=o(1),&\qquad 1\leq j\leq d_2,
    \end{align*}
    where $\{\frac{1}{R}\mu_j\}_{j=1}^{d_1}$ and $\{\frac{1}{R}\nu_j\}_{j=1}^{d_2}$ are the eigenvalues of matrices $\mA-\mC\mB^{-1}\mC^\top$ and $\mB$ respectively.
\end{lemma}
The proof follows from Lemma~\ref{rootsmishael} by a similar argument as in \cite{pmlr-v119-jin20e}  with $\{\mu_j\}_{j=1}^{d_1}$ and $\{\nu_j\}_{j=1}^{d_2}$ as the eigenvalues of matrices $R(\mA-\mC\mB^{-1}\mC^\top)$ and $R\mB$, respectively, and is thus omitted.

\end{document}